\documentclass[twoside,11pt]{article}

\usepackage{jmlr2e}
\usepackage{color}

\usepackage[utf8]{inputenc}

\usepackage{algorithm}
\usepackage{algorithmic}
\usepackage{amsmath}
\usepackage{subfigure}
\usepackage{wrapfig}

\usepackage{mathtools}

\usepackage{amsmath,amssymb, amsfonts}            

\ShortHeadings{Model-Free Trajectory-based Policy Optimization}{Akrour, Abdolmaleki, Abdulsamad, Peters, and Neumann}
\firstpageno{1}

\def\R{\mathbb{R}}
\def\EE{{\rm I\hspace{-0.50ex}E}}

\def\State{{\cal S}}
\def\Action{{\cal A}}
\def\Normal{{\cal N}}
\def\a{\boldsymbol{a}}
\def\s{\boldsymbol{s}}

\def\w{\boldsymbol{w}}

\def\dd{\mathrm{d}}
\def\KL{\mathrm{KL}}

\def\tr{\text{tr}}

\def\qfunc{Q-Function}

\DeclarePairedDelimiterX{\KLD}[2]{(}{)}{%
  #1\;\delimsize\|\;#2%
}
\newcommand{\KLM}{\KL\KLD}

\begin{document}

\title{Model-Free Trajectory-based Policy Optimization with Monotonic Improvement}

\author{\name Riad Akrour$^1$ \email riad@robot-learning.de\\
		\name Abbas Abdolmaleki$^2$ \email abbas.a@ua.pt\\
		\name Hany Abdulsamad$^1$ \email hany@robot-learning.de\\
		\name Jan Peters$^{1,3}$ \email jan@robot-learning.de\\
       \name Gerhard Neumann$^{1,4}$ \email geri@robot-learning.de\\  
		\addr $^1$CLAS/IAS, Technische Universit\"ate Darmstadt, Hochschulstr. 10, D-64289 Darmstadt, Germany\\	
		$^2$DeepMind, London N1C 4AG, UK\\
		$^3$Max Planck Institute for Intelligent Systems, Max-Planck-Ring 4, T\"ubingen, Germany\\
		$^4$L-CAS, University of Lincoln, Lincoln LN6 7TS, UK
       }

\editor{-}

\maketitle

\begin{abstract}
Many of the recent trajectory optimization algorithms alternate between linear approximation of the system dynamics around the mean trajectory and conservative policy update. One way of constraining the policy change is by bounding the Kullback-Leibler (KL) divergence between successive policies. These approaches already demonstrated great experimental success in challenging problems such as end-to-end control of physical systems. However, the linear approximation of the system dynamics can introduce a bias in the policy update and prevent convergence to the optimal policy. In this article, we propose a new model-free trajectory-based policy optimization algorithm with guaranteed monotonic improvement. The algorithm backpropagates a local, quadratic and time-dependent \qfunc~learned from trajectory data instead of a model of the system dynamics. Our policy update ensures exact KL-constraint satisfaction without simplifying assumptions on the system dynamics. We experimentally demonstrate on highly non-linear control tasks the improvement in performance of our algorithm in comparison to approaches linearizing the system dynamics. In order to show the monotonic improvement of our algorithm, we additionally conduct a theoretical analysis of our policy update scheme to derive a lower bound of the change in policy return between successive iterations.
\end{abstract}





\section{Introduction}
\label{sec:intro}
Trajectory Optimization methods based on stochastic optimal control \citep{Todorov2006,	Theodorou2009, Todorov09} have been very successful in learning high dimensional controls in complex settings such as end-to-end control of physical systems \citep{Levine14}. These methods are based on a time-dependent linearization of the dynamics model around the mean trajectory in order to obtain a closed form update of the policy as a Linear-Quadratic Regulator (LQR). This linearization is then repeated locally for the new policy at every iteration. However, this iterative process does not offer convergence guarantees as the linearization of the dynamics might introduce a bias and impede the algorithm from converging to the optimal policy. To circumvent this limitation, we propose in this paper a novel model-free trajectory-based policy optimization algorithm (MOTO) couched in the approximate policy iteration framework. At each iteration, a Q-Function is estimated locally around the current trajectory distribution using a time-dependent quadratic function. Afterwards, the policy is updated according to a new information-theoretic trust region that bounds the KL-divergence between successive policies in closed form. 

MOTO is well suited for high dimensional continuous state and action spaces control problems. The policy is represented by a time-dependent stochastic linear-feedback controller which is updated by a \qfunc\ propagated backward in time. We extend the work of \citep{Abdolmaleki15}, which was proposed in the domain of stochastic search (having no notion of state space nor that of sequential decisions), to that of sequential decision making and show that our policy class can be updated under a KL-constraint in closed form, when the learned \qfunc\ is a quadratic function of the state and action space. In order to maximize sample efficiency, we rely on importance sampling to reuse transition samples from policies of all time-steps and all previous iterations in a principled way. MOTO is able to solve complex control problems despite the simplicity of the \qfunc\ thanks to two key properties: i) the learned \qfunc\ is fitted to samples of the current policy, which ensures that the 
function is \textit{valid locally} and ii) the closed form update of the policy ensures that the KL-constraint is satisfied exactly irrespective of the number of samples or the non-linearity of the dynamics, which ensures that the \qfunc\ is \textit{used locally}.

The experimental section demonstrates that on tasks with highly non-linear dynamics MOTO outperforms similar methods that rely on a linearization of these dynamics. Additionally, it is shown on a simulated Robot Table Tennis Task that MOTO is able to scale to high dimensional tasks while keeping the sample complexity relatively low; amenable to a direct application to a physical system. 

In addition to the experimental validation previously reported in \cite{Akrour16}, we conduct a theoretical analysis of the policy update in Sec. \ref{sec:theo} and lower bound the increase in policy return between successive iterations of the algorithm. The resulting lower bound validates the use of an expected KL-constraint (Sec. \ref{sec:piPbm}) in a trajectory-based policy optimization setting for ensuring a monotonic improvement of the policy return. Prior theoretical studies reported similar results when the \textit{maximum} (over the state space) KL is upper bounded which is hard to enforce in practice \citep{Schulman15}. Leveraging standard trajectory optimization assumptions, we are able to extend the results when only the \textit{expected} KL under the state distribution of the previous policy is bounded.


\section{Notation}
Consider an undiscounted finite-horizon Markov Decision Process (MDP) of horizon $T$ with state space $\State = \R^{\dd_s}$ and action space $\Action = \R^{\dd_a}$. The transition function $p(\s_{t+1}|\s_t, \a_t)$, which gives the probability (density) of transitioning to state $\s_{t+1}$ upon the execution of action $\a_t$ in $\s_t$, is assumed to be time-independent; while there are $T$ time-dependent reward functions $r_t:\State \times \Action \mapsto \R$. A policy $\pi$ is defined by a set of time-dependent density functions $\pi_t$, where $\pi_t(\a| \s)$ is the probability of executing action $\a$ in state $\s$ at time-step $t$. The goal is to find the optimal policy $\pi^* = \{\pi^*_1, \dots, \pi^*_T\}$ maximizing the policy return $J(\pi) = \EE_{\s_1, \a_1, \dots}\left[\sum_{t=1}^Tr_t(\s_t, \a_t)\right]$, where the expectation is taken w.r.t. all the random variables $\s_t$ and $\a_t$ such that $\s_1\sim \rho_1$ follows the distribution of the initial state, $\a_t\sim \pi_t(.|\s_t)$ and $\s_{t+1}\sim p(\s_{t+1}
|\s_t, \a_t)$.

As is common in Policy Search \citep{Deisenroth2013}, our algorithm operates on a restricted class of parameterized policies $\pi_{\boldsymbol{\theta}}, \boldsymbol{\theta} \in \R^{\dd_\theta}$ and is an iterative algorithm comprising two main steps, \textit{policy evaluation} and \textit{policy update}. Throughout this article, we will assume that each time-dependent policy is parameterized by $\boldsymbol{\theta}_t = \{K_t,\boldsymbol{k}_t,\Sigma_t\}$ such that $\pi_{\boldsymbol{\theta}_t}$ is of linear-Gaussian form $\pi_{\boldsymbol{\theta}_t}(\a|\s) = \Normal(K_t s + \boldsymbol{k}_t, \Sigma_t)$, where the gain matrix $K_t$ is a $d_a\times d_s$ matrix, the bias term $\boldsymbol{k}_t$ is a $d_a$ dimensional column vector and the covariance matrix $\Sigma_t$, which controls the exploration of the policy, is of dimension $d_a\times d_a$; yielding a total number of parameters across all time-steps of $d_{\theta} = T(d_ad_s+\frac{1}{2}d_a(d_a+3))$. 

The policy at iteration $i$ of the algorithm is denoted by $\pi^i$ and following standard definitions, the Q-Function of $\pi^i$ at time-step $t$ is given by $Q^{i}_t(\s, \a) = \EE_{\s_t, \a_t, \dots}\left[\sum_{t'=t}^Tr_{t'}(\s_{t'}, \a_{t'})\right]$ with $(\s_t, \a_t)=(\s,\a)$ and $\a_{t'}\sim \pi^i_{t'}(.|\s_{t'}), \forall t' > t$. While the V-Function is given by $V^i_t(\s) = \EE_{\a\sim \pi_t(.|\s)}\left[Q^\pi_t(\s, \a)\right]$ and the Advantage Function by $A^i_t(\s,\a) = Q^i_t(\s,\a) - V_t^i(\s)$. Furthermore the state distribution at time-step $t$, related to policy $\pi^i$, is denoted by $\rho_t^i(\s)$. In order to keep the notations uncluttered, the time-step or the iteration number is occasionally dropped when a definition applies similarly for all time-steps or iteration number.


\section{Model-free Policy Update for Trajectory-based Policy Optimization}
\label{sec:policyUpdate}
MOTO alternates between policy evaluation and policy update. At each iteration $i$, the policy evaluation step generates a set of $M$ rollouts\footnote{A rollout is a Monte Carlo simulation of a trajectory according to $\rho_1$, $\pi$ and $p$ or the execution of $\pi$ on a physical system.} from the policy $\pi^i$ in order to estimate a (quadratic) Q-Function $\tilde{Q}^{i}$ (Sec. \ref{sec:qfunc}) and a (Gaussian) state distribution $\tilde{\rho}^i$ (Sec. \ref{sec:stateDist}). Using these quantities, an information-theoretic policy update is derived at each time-step that uses a KL-bound as a trust region to obtain the policy $\pi^{i+1}$ of the next iteration.
\subsection{Optimization Problem}
\label{sec:piPbm}
The goal of the policy update is to return a new policy $\pi^{i+1}$ that maximizes the Q-Function $\tilde{Q}^i$ in expectation under the state distribution $\tilde{p}_i$ of the previous policy $\pi^i$. In order to limit policy oscillation between iterations~\citep{Wagner11}, the KL w.r.t. $\pi^i$ is upper bounded. The use of the KL divergence to define the step-size of the policy update has already been successfully applied in prior work \citep{Peters2010, Levine14, Schulman15}. Additionally, we lower bound the entropy of $\pi^{i+1}$ in order to better control the reduction of exploration yielding the following non-linear program:
\begin{align}
& \underset{\pi}{\text{maximize}} \label{eq:maxQ}
& & \int\int \tilde{\rho}^i_t(\s) \pi(\a|\s)\tilde{Q}^{i}_t(\s,\a)\mathrm{d}a\mathrm{d}s,\\
& \text{subject to} \label{eq:klCst}
& & \EE_{s\sim \tilde{\rho}^i_t(\s)}\left[\KLM{\pi(.|\s)}{\pi_t^{i}(.|\s)}\right]\leq \epsilon,\\
&&& \EE_{s\sim \tilde{\rho}^i_t(\s)}\left[{\cal H}\left(\pi(.|\s)\right)\right] \geq \beta.
\label{eq:hCst}
\end{align}
The KL between two distributions $p$ and $q$ is given by $\KLM{p}{q} = \int p(x) \log  \frac{p(x)}{q(x)} \mathrm{d}x$ while the entropy ${\cal H}$ is given by ${\cal H} = -\int p(x) \log  p(x)\mathrm{d}x$. The step-size $\epsilon$ is a hyper-parameter of the algorithm kept constant throughout the iterations while $\beta$ is set according to the entropy of the current policy $\pi^i_t$, $\beta = \EE_{s\sim \tilde{\rho}^i_t(\s)}\left[{\cal H}\left(\pi^i_t(.|\s)\right)\right] - \beta_0$ and $\beta_0$ is the entropy reduction hyper-parameter kept constant throughout the iterations. 

Eq. \eqref{eq:maxQ} indicates that $\pi^{i+1}_t$ maximizes $\tilde{Q}^i_t$ in expectation under its own action distribution and the state distribution of $\pi^{i}_t$. Eq. \eqref{eq:klCst} bounds the average change in the policy to the step-size $\epsilon$ while Eq. \eqref{eq:hCst} controls the exploration-exploitation trade-off and ensures that the exploration in the action space (which is directly linked to the entropy of the policy) is not reduced too quickly. A similar constraint was introduced in the stochastic search domain by \citep{Abdolmaleki15}, and was shown to avoid premature convergence. This constraint is even more crucial in our setting because of the inherent \textit{non-stationarity} of the objective function being optimized at each iteration. The cause for the non-stationarity of the objective optimized at time-step $t$ in the policy update is twofold: i) updates of policies $\pi_{t'}$ with time-step $t' > t$ will modify in the next iteration of the algorithm $\tilde{Q}_t$ as a function of $s$ and $a$ and hence the optimization landscape as a function of the policy parameters, ii) updates of policies with time-step $t'<t$ will induce a change in the state distribution $\rho_t$. If the policy had unlimited expressiveness, the optimal solution of Eq. \eqref{eq:maxQ} would be to choose $\arg\max_a \tilde{Q}_t$ irrespective of $\rho_t$. However, due to the restricted class of the policy, any change in $\rho_t$ will likely change the optimization landscape including the position of the optimal policy parameter. Hence, Eq. \eqref{eq:hCst} ensures that exploration in action space is maintained as the optimization landscape evolves and avoids premature convergence.

\subsection{Closed Form Update}
\label{sec:piUpdate}
Using the method of Lagrange multipliers, the solution of the optimization problem in section \ref{sec:piPbm} is given by
\begin{align}
\label{eq:piUpdateGeneral}
\pi'_t(\a|\s) \propto \pi_t(\a|\s)^{\eta^*/(\eta^*+\omega^*)} \exp\left(\frac{\tilde{Q}_t(\s,\a)}{\eta^*+\omega^*}\right),
\end{align} 
with $\eta^*$ and $\omega^*$ being the optimal Lagrange multipliers related to the KL and entropy constraints respectively. Assuming that $\tilde{Q}_t(\s,\a)$ is of quadratic form in $a$ and $s$
\begin{align}
\label{eq:quadQ}
\tilde{Q}_t(\s,\a) = \frac{1}{2}\a^TQ_{aa}\a + \a^TQ_{as}\s+\a^T\boldsymbol{q}_{a}+q(\s),
\end{align}
with $q(\s)$ grouping all terms of $\tilde{Q}_t(\s,\a)$ that do not depend\footnote{Constant terms and terms depending on $s$ but not $a$ won't appear in the policy update. As such, and albeit we only refer in this article to $Q_t(\s,\a)$, the Advantage Function $A_t(\s,\a)$ can be used interchangeably in lieu of $Q_t(\s,\a)$ for updating the policy.} on $a$, then $\pi'_t(\a|\s)$ is again of linear-Gaussian form
\[\pi'_t(\a|\s) = \Normal(\a|FL\s+F\boldsymbol{f}, F(\eta^*+\omega^*)),\] 
such that the gain matrix, bias and covariance matrix of $\pi'_t$ are function of matrices $F$ and $L$ and vector $\boldsymbol{f}$ where
\begin{align*}
F = (\eta^* \Sigma_t^{-1}-Q_{aa})^{-1}, & & L = \eta^*\Sigma_t^{-1}K_t+Q_{as}, & & \boldsymbol{f} = \eta^*\Sigma_t^{-1}\boldsymbol{k}_t+\boldsymbol{q}_a.
\end{align*}
Note that $\eta \Sigma_t^{-1}-Q_{aa}$ needs to be invertible and positive semi-definite as it defines the new covariance matrix of the linear-Gaussian policy. For this to hold, either $Q_t(\s,.)$ needs to be concave in $a$ (i.e. $Q_{aa}$ is negative semi-definite), or $\eta$ needs to be large enough (and for any $Q_{aa}$ such $\eta$ always exists). A too large $\eta$ is not desirable as it would barely yield a change to the current policy (too small KL divergence) and could negatively impact the convergence speed. Gradient based algorithms for learning model parameters with a specific semi-definite shape are available \citep{Bhojanapalli15} and could be used for learning a concave $Q_t$. However, we found in practice that the resulting $\eta$ was always small enough~(resulting in a maximally tolerated KL divergence of $\epsilon$ between successive policies) while $F$ remains well defined, without requiring additional constraints on the nature of $Q_{aa}$.

\subsection{Dual Minimization}
\label{sec:dual}
The Lagrangian multipliers $\eta$ and $\omega$ are obtained by minimizing the convex dual function 
\[g_t(\eta, \omega) = \eta\epsilon - \omega\beta + (\eta+\omega)\int \tilde{\rho}_t(\s) \log\left( \int \pi_t(\a|\s)^{\eta/(\eta+\omega)}\exp\left(\tilde{Q}_t(\s,\a)/(\eta+\omega)\right) \dd \a \right) \dd \s.
\]
Exploiting the structure of the quadratic Q-Function $\tilde{Q}_t$ and the linear-Gaussian policy $\pi_t(\a|\s)$, the inner integral over the action space can be evaluated in closed form and the dual simplifies to
\[g(\eta, \omega) = \eta\epsilon_t-\omega\beta_t+\int \rho_t(\s) (\s^TM\s + \s^T\boldsymbol{m} + m_0) \dd \s.\]
The dual function further simplifies, by additionally assuming normality of the state distribution $\tilde{\rho}_t(\s) = \Normal(\s|\boldsymbol{\mu}_s, \Sigma_s)$, to the function 
\[g_t(\eta, \omega) = \eta\epsilon-\omega\beta+\boldsymbol{\mu}^T_s M\boldsymbol{\mu}_s + \tr(\Sigma_s M) + \boldsymbol{\mu}^T_s \boldsymbol{m} + m_0,\]
which can be efficiently optimized by gradient descent to obtain $\eta^*$ and $\omega^*$. The full expression of the dual function, including the definition of $M$, $\boldsymbol{m}$ and $m_0$ in addition to the partial derivatives $\frac{\partial g_t(\eta, \omega)}{\partial \eta}$ and $\frac{\partial g_t(\eta, \omega)}{\partial \omega}$ are given in Appendix \ref{sec:app_grad}.
\section{Sample Efficient Policy Evaluation}
\label{sec:policyEval}
The KL constraint introduced in the policy update gives rise to a non-linear optimization problem. This problem can still be solved in closed form for the class of linear-Gaussian policies, if the learned function $\tilde{Q}_t^i$ is quadratic in $\s$ and $\a$. The first subsection introduces the main supervised learning problem solved during the policy evaluation for learning $\tilde{Q}_t^i$ while the remaining subsections discuss how to improve its sample efficiency.
\subsection{The Q-Function Supervised Learning Problem}
\label{sec:qfunc}
In the remainder of the section, we will be interested in finding the parameter $\w$ of a linear model $\tilde{Q}_t^i = \langle\w, \phi(\s,\a)\rangle$, where the feature function $\phi$ contains a bias and all the linear and quadratic terms of $s$ and $a$, yielding $1 + (d_a + d_s)(d_a + d_s + 3) / 2$ parameters. $\tilde{Q}_t^i$ can subsequently be written as in Eq. \eqref{eq:quadQ} by extracting $Q_{aa}$, $Q_{as}$ and $\boldsymbol{q}_a$ from $\w$.

At each iteration $i$, $M$ rollouts are performed following $\pi^i$. Let us initially assume that $\tilde{Q}_t^i$ is learned only from samples ${\cal D}_t^i=\lbrace \s_t^{[k]}, \a_t^{[k]}, \s_{t+1}^{[k]}; k = 1..M \rbrace$ gathered from the execution of the M rollouts. The parameter $\w$ of $\tilde{Q}_t^i$ is learned by regularized linear least square regression
\begin{align}
\label{eq:lossQ}
\w = \underset{\w}{\arg\min} \frac{1}{M}\sum_{k=1}^M\bigl(\langle \w, \phi(\s_t^{[k]}, \a_t^{[k]})\rangle -  \hat{Q}_t^i(\s_t^{[k]}, \a_t^{[k]})\bigr)^2  + \lambda \w^T\w,
\end{align}
where the target value $\hat{Q}_t^i(\s^{[k]}, \a^{[k]})$ is a noisy estimate of the true value $Q_t^i(\s_t^{[k]}, \a_t^{[k]})$. We will distinguish two cases for obtaining the estimate $\hat{Q}_t^i(\s_t^{[k]}, \a_t^{[k]})$. 

\subsubsection{Monte-Carlo Estimate} This estimate is obtained by summing the future rewards
for each trajectory $k$, yielding $\hat{Q}_t^i(\s_t^{[k]}, \a_t^{[k]}) = \sum_{t' = t}^T r_{t'}(\s_{t'}^{[k]}, \a_{t'}^{[k]})$. This estimator is known to have no bias but high variance. The variance can be reduced by averaging over multiple rollouts, assuming we can reset to states $\s_t^{[k]}$. However, such an assumption would severely limit the applicability of the algorithm on physical systems.

\subsubsection{Dynamic Programming} In order to reduce the variance, this estimate  exploits the V-Function to reduce the noise of the expected rewards of time-steps $t'>t$ through the following identity 
\begin{align}
\label{eq:hatQ}
\hat{Q}_t^i(\s_t^{[k]}, \a_t^{[k]}) = r_t(\s_t^{[k]}, \a_t^{[k]}) + \hat{V}^i_{t+1}(\s_{t+1}^{[k]}),
\end{align} 
which is unbiased if $\hat{V}^i_{t+1}$ is. However, we will use for $\hat{V}^i_{t+1}$ an approximate V-Function $\tilde{V}^i_{t+1}$ learned recursively in time. This approximation might introduce a bias  which will accumulate as $t$ goes to 1. Fortunately, $\tilde{V}$ is not restricted by our algorithm to be of a particular class as it does not appear in the policy update. Hence, the bias can be reduced by increasing the complexity of the function approximator class. Nonetheless, in this article, a quadratic function will also be used for the V-Function which worked well in our experiments.

The V-Function is learned by first assuming that $\tilde{V}^i_{T+1}$ is the zero function.\footnote{Alternatively one could assume the presence of a final reward $R_{T+1}(\s_{t+1})$, as is usually formulated in control tasks \citep{Bertsekas1995}, to which $V^i_{T+1}$ could be initialized to.} Subsequently and recursively in time, the function $\tilde{V}^i_{t+1}$ and the transition samples in ${\cal D}_t^i$ are used to fit the parametric function $\tilde{V}_t^i$ by minimizing the loss $\sum_{k=1}^M\left(\hat{Q}_t^i(\s_t^{[k]}, \a_t^{[k]}) - \tilde{V}^i_t(\s_t^{[k]}) \right)^2$.

In addition to reducing the variance of the estimate $\hat{Q}_t^i(\s_t^{[k]}, \a_t^{[k]})$, the choice of learning a V-Function is further justified by the increased possibility of reusing sample transitions from all time-steps and previous iterations.
\subsection{Sample Reuse}
\label{sec:importance}
In order to improve the sample efficiency of our approach, we will reuse samples from different time-steps and iterations using importance sampling. Let the expected loss which $\tilde{Q}_t^i$ minimizes under the assumption of an infinite number of samples be 
\begin{align*}
\w = \arg\min_{\w}\EE[\ell_t^i(\s,\a,\s';\w)],
\end{align*} 
where the loss $\ell_t^i$ is the inner term within the sum in Eq. \eqref{eq:lossQ}; the estimate $\hat{Q}_t^i(\s_t^{[k]}, \a_t^{[k]})$ is taken as in Eq. \eqref{eq:hatQ} and the expectation is with respect to the current state $\s\sim\rho_t^i$, the action $\a\sim\pi_t^i(.|\s)$ and the next state $\s'\sim p(.|\s,\a)$. 

\subsubsection{Reusing samples from different time-steps}
To use transition samples from all time-steps when learning $\tilde{Q}_t^i$, we rely on importance sampling, where the importance weight (IW) is given by the ratio between the state-action probability of the current time-step $z^i_t(\s,\a) = \rho^i_t(\s)\pi^i_t(\a|\s)$ divided by the time-independent state-action probability of $\pi^i$ given by $z^i(\s,\a) = \frac{1}{T}\sum_{t=1}^Tz^i_t(\s,\a)$. The expected loss minimized by $\tilde{Q}_t^i$ becomes
\begin{align}
\label{eq:iw}
\min_{\w}\EE \left[\frac{z_t^i(\s,\a)}{z^i(\s,\a)} \ell_t^i(\s,\a,\s';\w)\mid (\s,\a)\sim z^i(\s,\a)\right].
\end{align}
Since the transition probabilities are not time-dependent they cancel out from the IW. Upon the computation of the IW, weighted least square regression is used to minimize an empirical estimate of \eqref{eq:iw} for the data set ${\cal D}^i = \cup_{t=1}^T {\cal D}^i_t$. Note that the (numerator of the) IW needs to be recomputed at every time-step for all samples $(\s,\a)\in {\cal D}^i$. Additionally, if the rewards are time-dependent, the estimate $\hat{Q}_t^i(\s_t^{[k]}, \a_t^{[k]})$ in Eq. \eqref{eq:hatQ} needs to be recomputed with the current time-dependent reward, assuming the reward function is known. 

\subsubsection{Reusing samples from previous iterations}
Following a similar reasoning, at a given time-step $t$, samples from previous iterations can be reused for learning $\tilde{Q}_t^i$. In this case, we have access to the samples of the state-action distribution $z^{1:i}_t(\s,\a) \propto \sum_{j=1}^i z^j_t(\s,\a)$. The computation of $z^{1:i}_t$ requires the storage of all previous policies and state distributions. Thus, we will in practice limit ourselves to the $K$ last iterations.

Finally, both forms of sample reuse will be combined for learning $\tilde{Q}_t^i$ under the complete data set up to iteration $i$, ${\cal D}^{1:i} = \cup_{j=1}^i{\cal D}^j$ using weighted least square regression where the IW are given by $z^i_t(\s,\a)/z^{1:i}(\s,\a)$ with $z^{1:i}(\s,\a) \propto \sum_{t=1}^T z^{1:i}_t(\s,\a)$.

\subsection{Estimating the State Distribution}
\label{sec:stateDist}
To compute the IW, the state distribution at every time-step $\rho^i_t$ needs to be estimated. Since $M$ rollouts are sampled for every policy $\pi^i$ only $M$ state samples are available for the estimation of $\rho^i_t$, necessitating again the reuse of previous samples to cope with higher dimensional control tasks.

\subsubsection{Forward propagation of the state distribution}
The first investigated solution for the estimation of the state distribution is the propagation of the estimate $\tilde{\rho}^i_t$ forward in time. Starting from $\tilde{\rho}^i_1$ which is identical for all iterations, importance sampling is used to learn $\tilde{\rho}^i_{t+1}$ with $t > 1$ from samples $(\s_{t}, \a_{t}, \s_{t+1})\in{\cal D}^{1:i}_{t}$ by weighted maximum-likelihood; where each sample $\s_{t+1}$ is weighted by $z^i_{t}(\s_{t}, \a_{t})/z_{t}^{1:i}(\s_{t}, \a_{t})$. And the computation of this IW only depends on the previously estimated state distribution $\tilde{\rho}^i_{t}$. 
In practice however, the estimate $\tilde{\rho}^i_t$ might entail errors despite the use of all samples from past iterations, which are propagated forward leading to a degeneracy of the number of effective samples in latter time-steps.

\subsubsection{State distribution of a mixture policy}
The second considered solution for the estimation of $\tilde{\rho}^i_{t}$ is heuristic but behaved better in practice. It is based on the intuition that the KL constraint of the policy update will yield state distributions that are close to each other (see Sec. \ref{sec:theo} for a theoretical justification of the closeness in state distributions) and state samples from previous iterations can be reused in a simpler manner. Specifically, $\tilde{\rho}_t^i$ will be learned from samples of the mixture policy $\pi^{1:i} \propto \sum_{j = 1}^i \gamma^{i-j}\pi^{j}$ which selects a policy from previous iterations with an exponentially decaying (w.r.t. the iteration number) probability and executes it for a whole rollout. In practice, the decay factor $\gamma$ is selected according to the dimensionality of the problem, the number of samples per iterations $M$ and the KL upper bound $\epsilon$ (intuitively, a smaller $\epsilon$ yields closer policies and henceforth more reusable samples). The estimated state distribution $\tilde{\rho}_t^i$ is learned as a Gaussian distribution by weighted maximum likelihood from samples of ${\cal D}_t^{1:i}$ where a 
sample of iteration $j$ is weighted by $\gamma^{i-j}$.

\begin{algorithm}[tb]
   \caption{{\small Model-Free Trajectory-based Policy Optimization} (MOTO)}
   \label{alg:clops}
\begin{algorithmic}
   \STATE {\bfseries Input:} Initial policy $\pi^0$, number of trajectories per iteration M, step-size $\epsilon$ and entropy reduction rate $\beta_0$
   \STATE {\bfseries Output:} Policy $\pi^{N}$ 
   \FOR{$i=0$ {\bfseries to} $N-1$}
   \STATE Sample M trajectories from $\pi^i$		
		\FOR{$t=T$ {\bfseries to} $1$}
		\STATE Estimate state distribution $\tilde{\rho}^i_t$ \hfill (Sec. \ref{sec:stateDist})       
        \STATE Compute IW for all $(\s,a,s') \in {\cal D}^{1:i}$\hfill (Sec. \ref{sec:importance})		
		\STATE Estimate the Q-Function $\tilde{Q}_t^i$ \hfill (Sec. \ref{sec:qfunc})
		\STATE Optimize: $(\eta^*, \omega^*) = \arg\min g^i_t(\eta, \omega)$ \hfill (Sec. \ref{sec:dual})
		\STATE Update $\pi_t^{i+1}$ using $\eta^*$, $\omega^*$,$\tilde{\rho}^i_t$ and  $\tilde{Q}_t^i$ \hfill (Sec. \ref{sec:piUpdate})
		\ENDFOR   
   \ENDFOR
\end{algorithmic}
\end{algorithm}

\subsection{The MOTO Algorithm}
MOTO is summarized in Alg. \ref{alg:clops}. The innermost loop is split between policy evaluation (Sec. \ref{sec:policyEval}) and policy update (Sec. \ref{sec:policyUpdate}). For every time-step $t$, once the state distribution $\tilde{\rho}_t^i$ is estimated, the IWs of all the transition samples are computed and used to learn the Q-Function (and the V-Function using the same IWs, if dynamic programming is used when estimating the Q-Function), concluding the policy evaluation part.
Subsequently, the components of the quadratic model $\tilde{Q}_t^i$ that depend on the action are extracted and used to find the optimal dual parameters $\eta^*$ and $\omega^*$ that are respectively related to the KL and the entropy constraint, by minimizing the convex dual function $g_t^i$ using gradient descent. The policy update then uses $\eta^*$ and $\omega^*$ to return the new policy $\pi_{t+1}$ and the process is iterated.

In addition to the simplification of the policy update, the rationale behind the use of a local quadratic approximation for $Q_t^i$ is twofold: i) since $Q_t^i$ is only optimized locally (because of the KL constraint), a quadratic model would potentially contain as much information as a Hessian matrix in a second order gradient descent setting ii) If $\tilde{Q}_t$ in Eq. \eqref{eq:piUpdateGeneral} is an arbitrarily complex model then it is common that $\pi'_t$, of linear-Gaussian form, is fit by weighted maximum-likelihood \citep{Deisenroth2013}; it is clear though from Eq. \eqref{eq:piUpdateGeneral} that however complex $\tilde{Q}_t(\s,\a)$ is, if both $\pi_t$ and $\pi'_t$ are of linear-Gaussian form then there exist a quadratic model that would result in the same policy update. Additionally, note that $\tilde{Q}_t$ is not used when learning $\tilde{Q}_{t-1}$ (sec. \ref{sec:qfunc}) and hence the bias introduced by $\tilde{Q}_t$ will not propagate back. For these reasons, we think that choosing a more complex class for $\tilde{Q}_t$ than that of quadratic functions might not necessarily lead to an improvement of the resulting policy, for the class of linear-Gaussian policies. 


\section{Monotonic Improvement of the Policy Update}
\label{sec:theo}
We analyze in this section the properties of the constrained optimization problem solved during our policy update. \cite{Kakade02} showed that in the approximate policy iteration setting, a monotonic improvement of the policy return can be obtained if the successive policies are close enough. While in our algorithm the optimization problem defined in Sec. \ref{sec:piPbm} bounds the expected policy KL under the state distribution of the current iteration $i$, it does not tell us how similar the policies are under the new state distribution and a more careful analysis needs to be conducted.  

The analysis we present builds on the results of \cite{Kakade02} to lower-bound the change in policy return $J(\pi^{i+1}) - J(\pi^{i})$ between the new policy $\pi^{i+1}$ (solution of the optimization problem defined in Sec. \ref{sec:piPbm}) and the current policy $\pi^{i}$. Unlike \cite{Kakade02}, we enforce closeness between successive policies with a KL constraint instead of by mixing $\pi^{i+1}$ and $\pi^{i}$. Related results were obtained when a KL constraint is used in \cite{Schulman15}. Our main contribution is to extend these results to the trajectory optimization setting with continuous states and actions and where the expected KL between the policies is bounded instead of the maximum KL over the state space (which is hard to achieve in practice). 

In what follows, $p$ and $q$ denote the next policy $\pi^{i+1}$ and the current policy $\pi^i$ respectively. We will denote the state distribution and policy at time-step $t$ by $p_t$ and $p_t(.|\s)$ respectively (and similarly for $q$). First, we start by writing the difference between policy returns in term of advantage functions.

\begin{lemma}
\label{th:perfDiff}
For any two policies $p$ and $q$, and where $A_t^q$ denotes the advantage function at time-step $t$ of policy $q$, the difference in policy return is given by  \[J(p) - J(q) = \sum_{t=1}^{T}\EE_{\s\sim p_t, \a \sim p_t(.|\s)}\left[A_t^q(\s,\a)\right].\]
\end{lemma}

The proof of Lemma \ref{th:perfDiff} is given by the proof of Lemma 5.2.1 in \citep{Kakade03}. Note that Lemma \ref{th:perfDiff} expresses the change in policy return in term of expected advantage under the current state distribution while we optimize the advantage function under the state distribution of policy $q$, which is made apparent in Lemma \ref{th:perf_kl}.

\begin{lemma}
\label{th:perf_kl}
Let $\epsilon_t = \KLM{p_t}{q_t}$ be the KL divergence between state distributions $p_t(.)$ and $q_t(.)$ and let $\delta_t = \max_s|\EE_{\a\sim p_t(.|\s)}[A_t^q(\s,\a)]|$, then for any two policies p and q we have \[J(p) - J(q) \geq \sum_{t=1}^{T}\EE_{\s \sim q_t, \a \sim p_t(. | \s)}\left[A_t^q(\s,\a)\right] - 2\sum_{t=1}^{T}\delta_t\sqrt{\frac{\epsilon_t}{2}}.\]
\end{lemma}

\begin{proof}
\begin{align*}
\EE_{s\sim p_t, a \sim p_t(.|\s)}\left[A_t^q(\s_t,\a_t)\right] &= \int p_t(\s) \int p_t(\a_t | \s_t)  A_t^q(\s_t,\a_t),\\
&= \!\begin{multlined}[t] 
\int q_t(\s) \int p_t(\a_t | \s_t)  A_t^q(\s_t,\a_t) \\+ \int (p_t(\s)-q_t(\s)) \int p_t(\a_t | \s_t)  A_t^q(\s_t,\a_t), 
\end{multlined}\\
&\geq \EE_{s \sim q_t, a \sim p_t(. | \s)}\left[A_t^q(\s,\a)\right] - \delta_t \int (p_t(\s)-q_t(\s)),\\
&\geq \EE_{s \sim q_t, a \sim p_t(. | \s)}\left[A_t^q(\s,\a)\right] - 2\delta_t \frac{1}{2}\int |p_t(\s)-q_t(\s)|,\\
&\geq \EE_{s \sim q_t, a \sim p_t(. | \s)}\left[A_t^q(\s,\a)\right] - 2\delta_t \sqrt{\frac{1}{2}\KLM{p_t}{q_t}}. \tag{Pinsker's inequality}
\end{align*}
Summing over the time-steps and using Lemma \ref{th:perfDiff} completes the proof.
\end{proof}

Lemma \ref{th:perf_kl} lower-bounds the change in policy return by the advantage term optimized during the policy update and a negative change that quantifies the change in state distributions between successive policies. The core of our contribution is given by Lemma~\ref{th:kl_state} which relates the change in state distribution to the expected KL constraint between policies of our policy update.

\begin{lemma}
\label{th:kl_state}
If for every time-step, the state distributions $p_t$ and $q_t$  are Gaussian and the policies $p_t(.|\s_t)$ and $q_t(.|\s_t)$ are linear-Gaussian and if $\EE_{s\sim q_t}\left[\KLM{p_t(.|\s)}{q_t(.|\s)}\right]\leq \epsilon$ for every time-step then  $\KLM{p_t}{q_t} = \mathcal{O}(\epsilon)$ as $\epsilon \to 0$ for every time-step.
\end{lemma}

\begin{proof}
We will demonstrate the lemma by induction noting that for $t=1$ the state distributions are identical and hence their KL is zero. Assuming $\epsilon_t = \KLM{p_t}{q_t} = \mathcal{O}(\epsilon)$ as $\epsilon \to 0$, let us compute the KL between state distributions for $t+1$
\begin{align}
\KLM{p_{t+1}}{q_{t+1}} &= \int p_{t+1}(\s') \log \frac{p_{t+1}(\s')}{q_{t+1}(\s')},\notag\\
&\leq \iiint p_t(\s, \a)  p(\s'| \a, \s)  \log \frac{p_t(\s, \a)p(\s'| \a, \s)}{q_t(\s, \a)p(\s'| \a, \s)}\tag{log sum inequality},\\
 &= \int p_t(\s')\int p_t(\a|\s') \log \frac{p_t(\s')p_t(\a|\s')}{q_t(\s')q_t(\a|\s')},\notag\\
 &= \epsilon_t + \EE_{s\sim p_{t}}[\KLM{p_{t}(.|\s)}{q_{t}(.|\s)}]. \label{eq:state_plus}
\end{align}
Hence we have bounded the KL between state distributions at $t+1$ by the KL between state distributions and the expected KL between policies of the previous time-step $t$. Now we will express the KL between policies under the new state distributions, given by $\EE_{s\sim p_{t}}[\KLM{p_{t}(.|\s)}{q_{t}(.|\s)}]$, in terms of KL between policies under the previous state distribution, $\EE_{s\sim q_t}\left[\KLM{p_t(.|\s)}{q_t(.|\s)}\right]$ which is bounded during policy update by $\epsilon$, and $\KLM{p_t}{q_t}$. To do so, we will use the assumption that the state distribution and the policy are Gaussian and linear-Gaussian. The complete demonstration is given in Appendix \ref{sec:bounds}, and we only report the following result
\begin{align}
\EE_{s\sim p_{t}}[\KLM{p_{t}(.|\s)}{q_{t}(.|\s)}] &\leq 2\epsilon \left(3 \epsilon_t +  d_s + 1\right). \label{eq:kl_diff}
\end{align}
It is now easy to see that the combination of \eqref{eq:state_plus} and \eqref{eq:kl_diff} together with the induction hypothesis yields $\KLM{p_{t+1}}{q_{t+1}} = \mathcal{O}(\epsilon)$ as $\epsilon \to 0$.
\end{proof}

Finally, the combination of Lemma \ref{th:perf_kl} and Lemma \ref{th:kl_state} results in the following theorem, lower-bounding the change in policy return.

\begin{theorem}
\label{th:main}
If for every time-step the state distributions $p_t$ and $q_t$  are Gaussian and the policies $p_t(.|\s_t)$ and $q_t(.|\s_t)$ are linear-Gaussian and if $\EE_{s\sim q_t}\left[\KLM{p_t(.|\s)}{q_t(.|\s)}\right]\leq \epsilon$ for every time-step then \[J(p) - J(q) \geq \sum_{t=1}^{T}\EE_{s \sim q_t, a \sim p_t(. | \s)}\left[A_t^q(\s,\a)\right] - \sum_{t=1}^{T}\delta_t\mathcal{O}(\sqrt{\epsilon}).\]
\end{theorem}

Theorem \ref{th:main} shows that we are able to obtain similar bounds than those derived in \citep{Schulman15} for our continuous state-action trajectory optimization setting with a bounded KL policy update in expectation under the previous state distribution. While, it is not easy to apply Theorem \ref{th:main} in practice to choose an appropriate step-size $\epsilon$ since $A_t^q(\s,\a)$ is generally only known approximately, Theorem \ref{th:main} still shows that our constrained policy update will result in small changes in the overall behavior of the policy between successive iterations which is crucial in the approximate RL setting.


\section{Related Work}
\label{sec:soa}
In the Approximate Policy Iteration scheme \citep{Szepesvari10}, policy updates can potentially decrease the expected reward, leading to policy oscillations \citep{Wagner11}, unless the updated policy is 'close' enough to the previous one \citep{Kakade02}. Bounding the change between $\pi^{i}$ and $\pi^{i+1}$ during the policy update step is thus a well studied idea in the Approximate Policy Iteration literature. Already in 2002, \citeauthor{Kakade02} proposed the Conservative Policy Iteration (CPI) algorithm where the new policy $\pi^{i+1}$ is obtained as a mixture of $\pi^{i}$ and the greedy policy w.r.t. $Q^i$. The mixture parameter is chosen such that a lower bound of $J(\pi^{i+1}) - J(\pi^{i})$ is positive and improvement is guaranteed. However, convergence was only asymptotic and in practice a single policy update would require as many samples as other algorithms would need to find the optimal solution \citep{Pirotta13Safe}. \cite{Pirotta13Safe} refined the lower bound of CPI by adding an additional term capturing the closeness between policies (defined as the matrix norm of the difference between the two policies), resulting in a more aggressive updates and better experimental results. However, both approaches only considered discrete action spaces. \cite{Pirotta13Adaptive} provide an extension to continuous domains but only for single dimensional actions.

When the action space is continuous, which is typical in e.g. robotic applications, using a stochastic policy and updating it under a KL constraint to ensure 'closeness' of successive policies has shown several empirical successes \citep{Daniel2012, Levine14Model, Schulman15}. However, only an empirical sample estimate of the objective function is generally optimized \citep{Peters2010, Schulman15}, which typically requires a high number of samples and precludes it from a direct application to physical systems. The sample complexity can be reduced when a model of the dynamics is available \citep{Levine14Model} or learned \citep{Levine14}. In the latter work, empirical evidence suggests that good policies can be learned on high dimensional continuous state-action spaces with only a few 
hundred episodes. The counter part being that time-dependent dynamics are assumed to be linear, which is a simplifying assumption in many cases. Learning more sophisticated models using for example Gaussian Processes was experimented by \cite{Deisenroth2011} and \cite{Pan14} in the Policy Search and Trajectory Optimization context, but it is still considered to be a challenging task, see \cite{Deisenroth2013}, chapter 3.  

The policy update in Eq. \eqref{eq:piUpdateGeneral} resembles that of \citep{Peters2010, Daniel2012} with three main differences. First, without the assumption of a quadratic Q-Function, an additional weighted maximum likelihood step is required for fitting $\pi^{i+1}$ to weighted samples as in the r.h.s of Eq. \eqref{eq:piUpdateGeneral}, since this policy might not be of the same policy class. As a result, the KL between $\pi^{i}$ and $\pi^{i+1}$ is no longer respected. Secondly, we added an entropy constraint in order to cope with the inherent non-stationary objective function maximized by the policy (Eq. \ref{eq:maxQ}) and to ensure that exploration is sustained, resulting in better quality policies. Thirdly, their sample based optimization algorithm requires the introduction of a number of dual variables typically scaling at least linearly with the dimension of the state space, while we only have to 
optimize over two dual variables irrespective of the state space. 

Most trajectory optimization methods are based on stochastic optimal control. These methods linearize the system dynamics and update the policy in closed form as a LQR. Instances of such algorithms are for example iLQG \citep{Todorov2006}, DDP \citep{Theodorou2010a}, AICO \citep{Toussaint2009} and its more robust variant \citep{Ruckert14} and the trajectory optimization algorithm used in the GPS algorithm \citep{Levine14}. These methods share the same assumptions as MOTO for $\rho_t^i$ and $\pi^i_t$ respectively considered to be of Gaussian and linear-Gaussian form. These methods face issues in maintaining the stability of the policy update and, similarly to MOTO, introduce additional constraints and regularizers to their update step. DDP, iLQG and AICO regularize the update by introducing a damping term in the matrix inversion step, while GPS uses a KL bound on successive trajectory distributions. However, as demonstrated in Sec. \ref{sec:xp}, the quadratic approximation of the Q-Function performed by MOTO seems to be empirically less detrimental to the quality of the policy update than the linearization of the system dynamics around the mean trajectory performed by related approaches. 


\section{Experimental Validation}
\label{sec:xp}
MOTO is experimentally validated on a set of multi-link swing-up tasks and on a robot table tennis task. The experimental section aims at analyzing the proposed algorithm from four different angles: i) the quality of the returned policy comparatively to state-of-the-art trajectory optimization algorithms, ii) the effectiveness of the proposed variance reduction and sample reuse schemes, iii) the contribution of the added entropy constraint during policy updates in finding better local optima and iv) the ability of the algorithm to scale to higher dimensional problems. The experimental section concludes with a comparison to TRPO \citep{Schulman15}, a state-of-the-art reinforcement learning algorithm that bounds the KL between successive policies; showcasing settings in which the time-dependent linear-Gaussian policies used by MOTO are a suitable alternative to neural networks.

\subsection{Multi-link Swing-up Tasks}
\label{sec:swing}
\begin{figure*}[t!]
    \centering
    \subfigure[]{\includegraphics[width=0.31\textwidth]{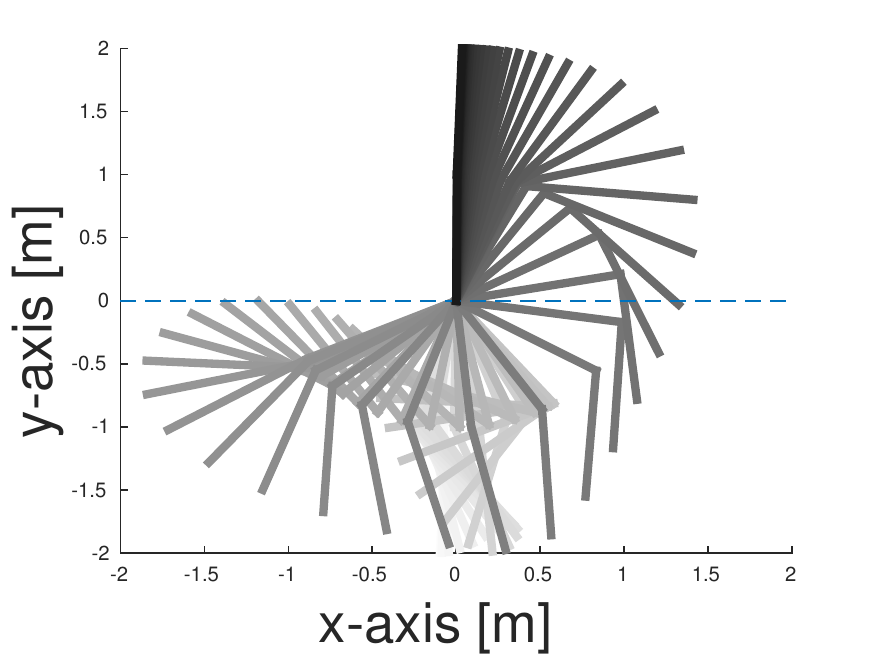}}
    ~ 
            \subfigure[]{\includegraphics[width=0.31\textwidth, height = 40mm]{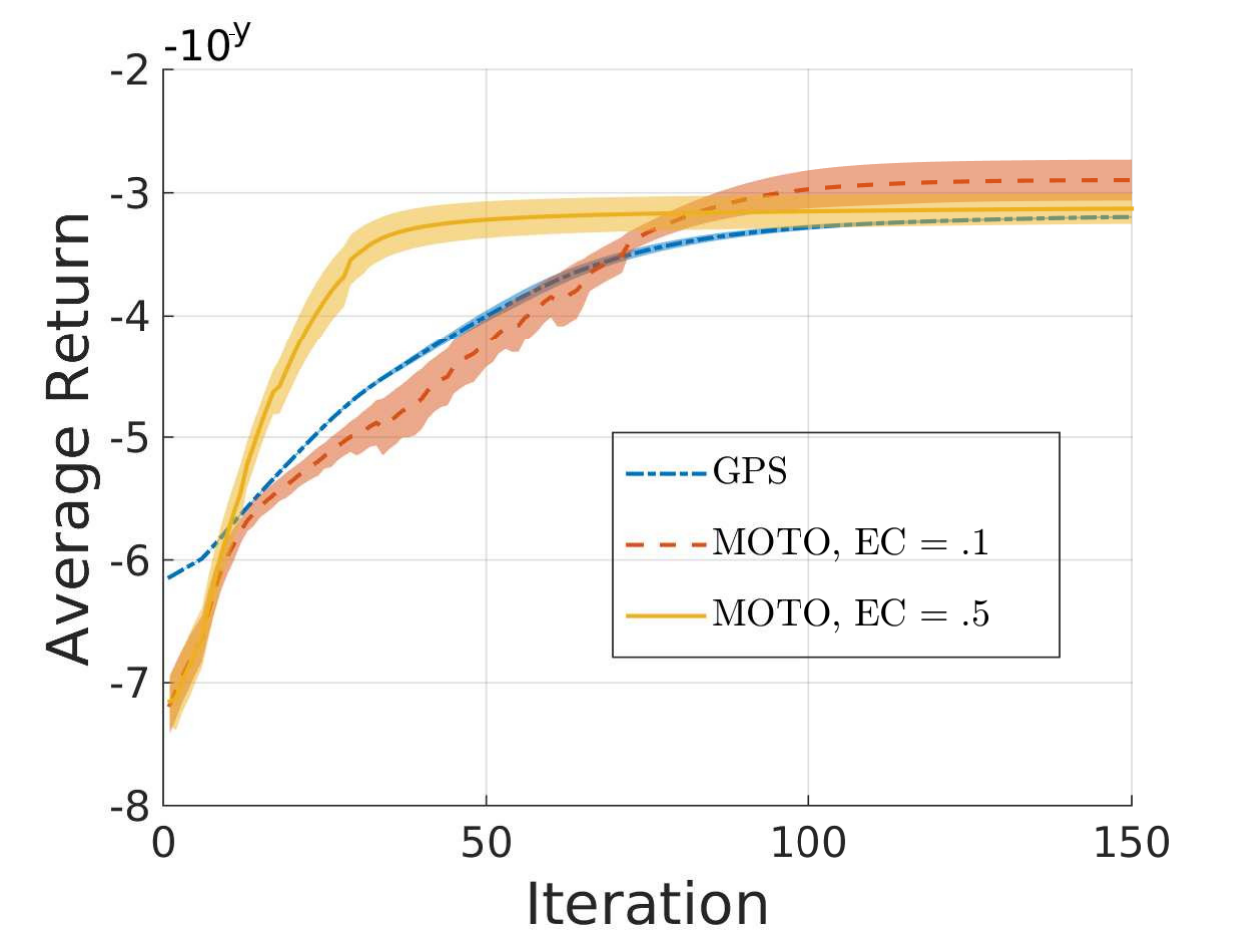}}
    ~ 
       \subfigure[]{\includegraphics[width=0.31\textwidth]{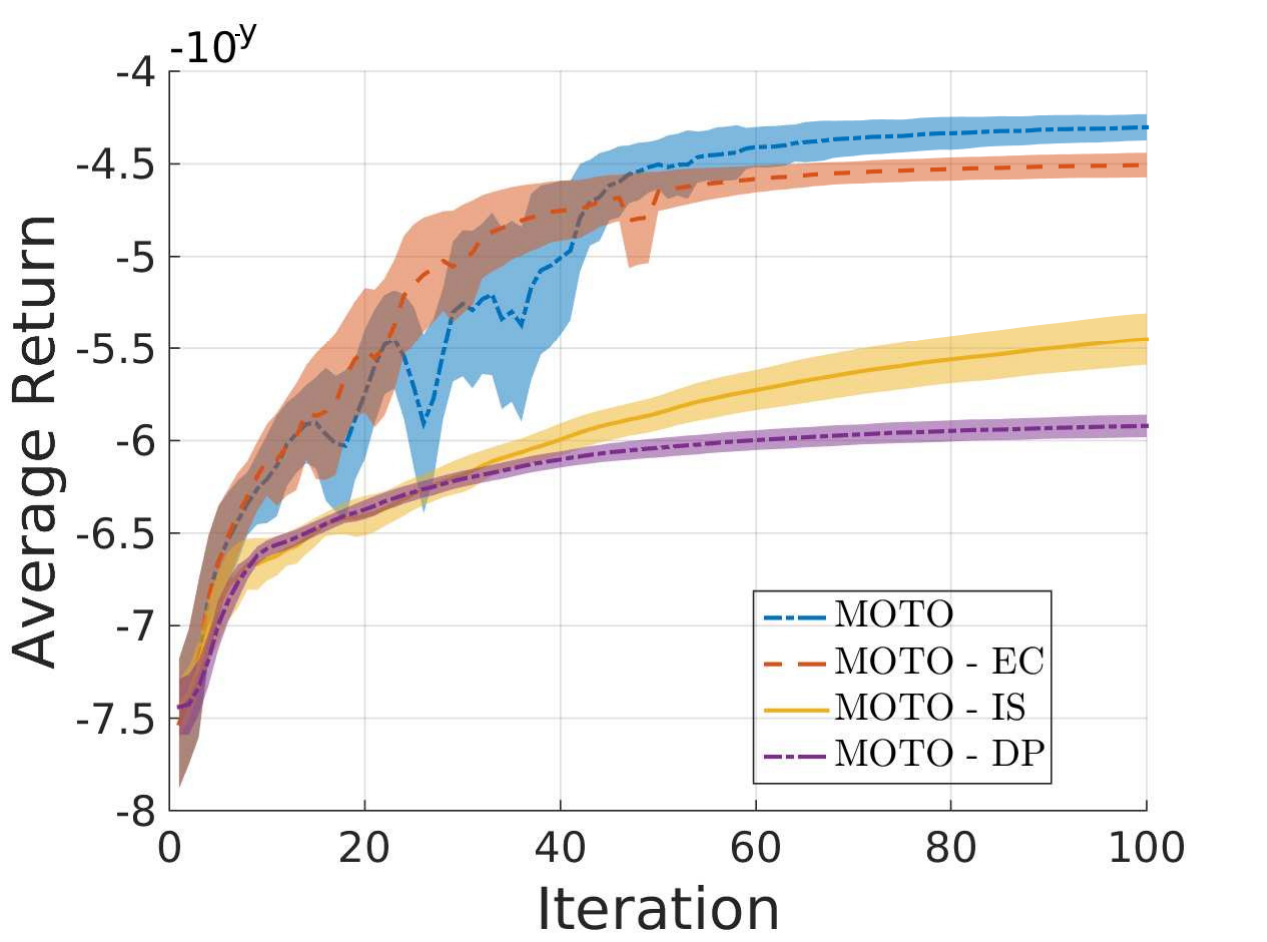}}
    \caption{ \small a) Double link swing-up policy found by MOTO. b) Comparison between GPS and MOTO on the double link swing-up task (different torque limits and state costs are applied compared to c) and f). 
    c) MOTO and its variants on the double link swing-up task: MOTO without the entropy constraint (EC), importance sampling (IS ) or dynamic programming (DP). All plots are averaged over 15 runs. \label{fig:swingTaska}}
 \end{figure*}

\begin{figure*}[t!]
    \centering
            \subfigure[]{\includegraphics[width=0.32\textwidth]{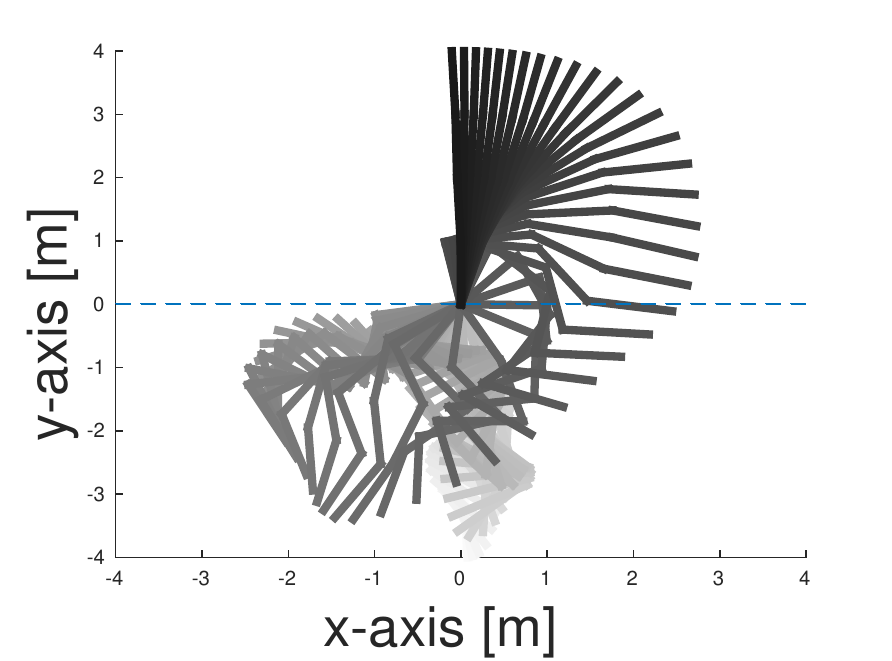}}
    \subfigure[]{\includegraphics[width=0.32\textwidth]{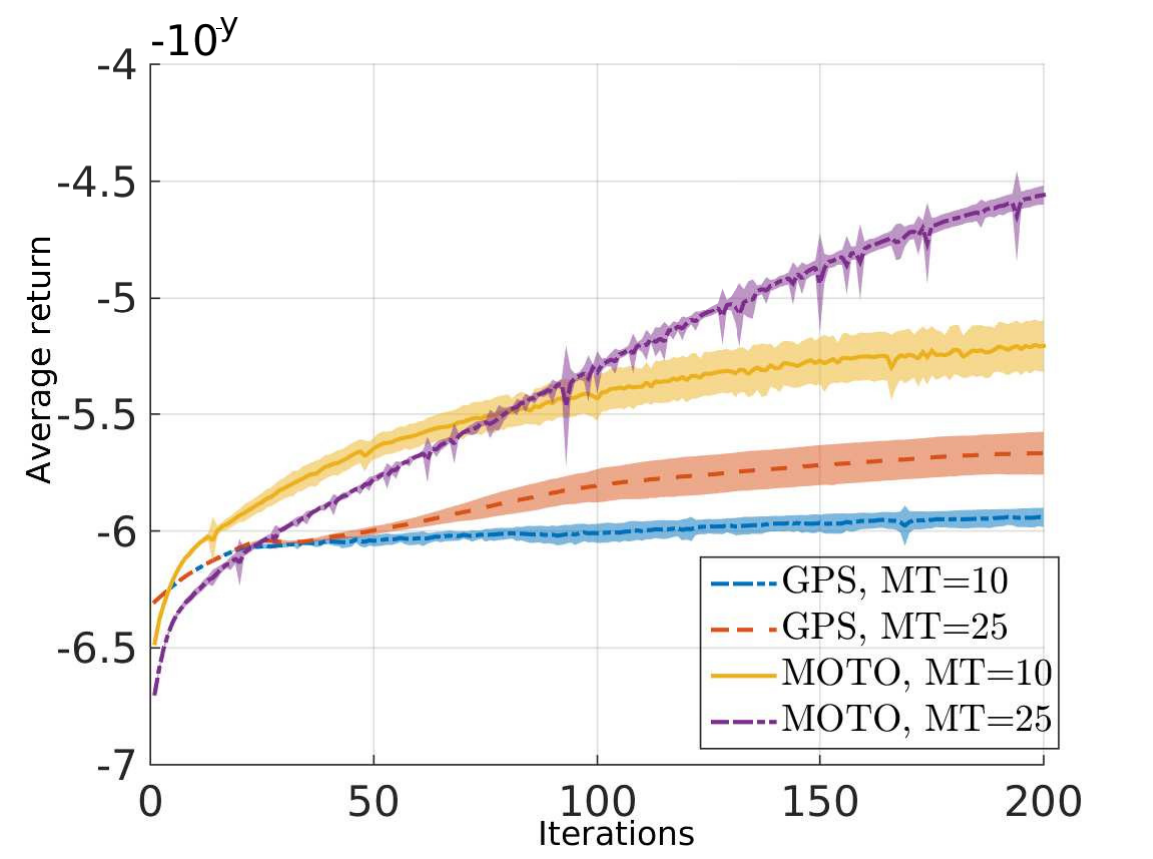}}
 \subfigure[]{\includegraphics[width=0.32\textwidth,  height=4cm]{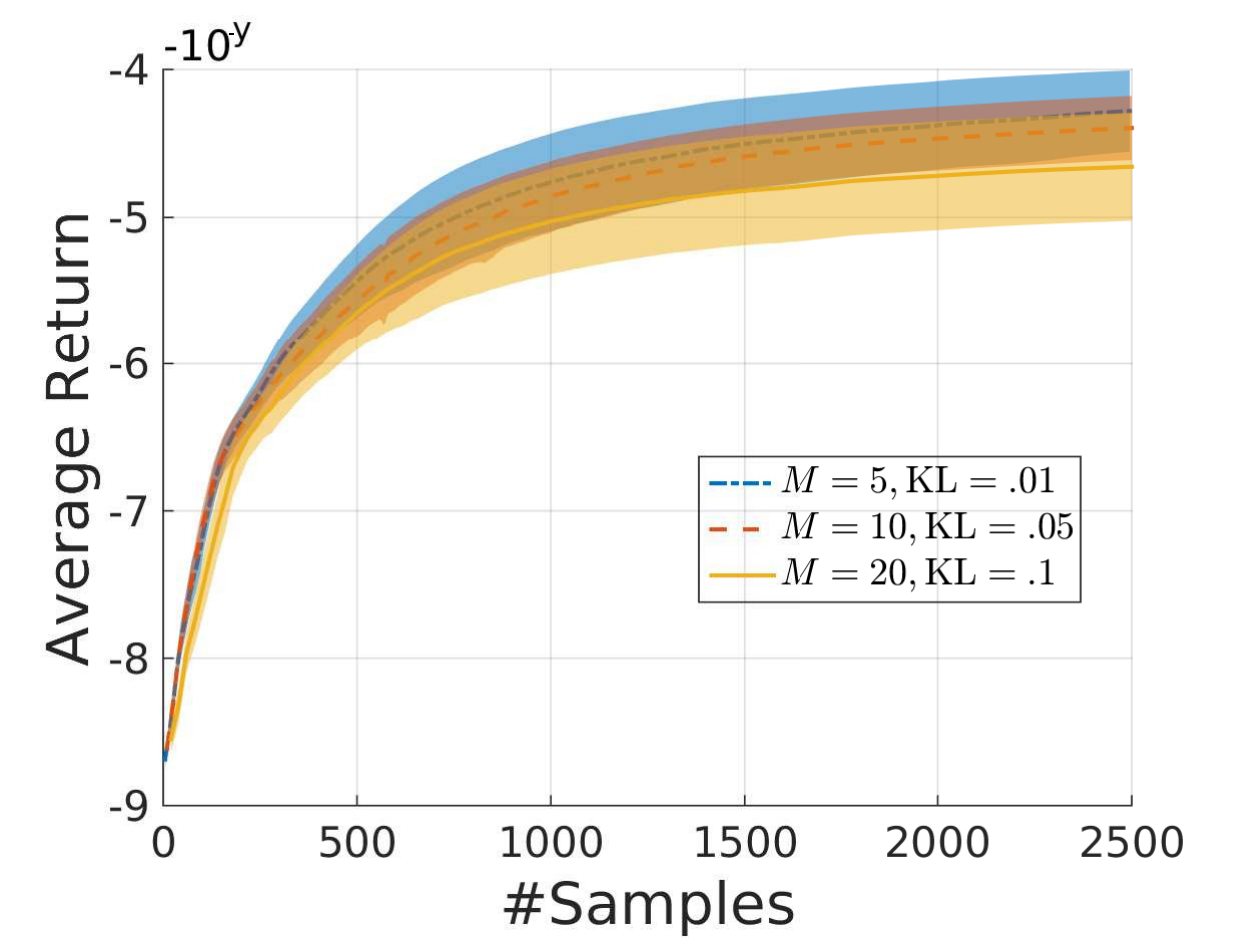}}

~ 
    ~ 
\caption{ \small a) Quad link swing-up policy found by MOTO. 
    b) Comparison between GPS and MOTO on the quad link swing-up task with restricted joint limits and two different torque limits. c) MOTO on the double link swing-up task for 
    varying number of rollouts per episode and step-sizes.  All plots are averaged over 15 runs.}\label{fig:swingTaskb}
\end{figure*}

A set of swing-up tasks involving a multi-link pole with respectively two and four joints (Fig. \ref{fig:swingTaska}.a and \ref{fig:swingTaskb}.a) is considered in this section. 
The set of tasks includes several variants with different torque and joint limits, introducing additional non-linearities in the dynamics and resulting in 
more challenging control problems for trajectory optimization algorithms based on linearizing the dynamics. The state space consists of the joint positions and 
joint velocities while the control actions are the motor torques. In all the tasks, the reward function is split between an action cost and a state cost. 
The action cost is constant throughout the time-steps while the state cost is time-dependent and is equal to zero for all but the 20 last time-steps. 
During this period, a quadratic cost penalizes the state for not being the null vector, i.e. having zero velocity and reaching the upright position. 
Examples of successful swing-ups learned with MOTO are depicted in Fig. \ref{fig:swingTaska}.a and \ref{fig:swingTaskb}.a.\\

MOTO is compared to the trajectory optimization algorithm proposed in \cite{Levine14}, 
that we will refer to as GPS.\footnote{This is a slight abuse of notation as the GPS algorithm of \citep{Levine14} additionally feeds the optimized trajectory to an upper level policy. 
However, in this article, we are only interested in the trajectory optimization part.} We chose to compare MOTO and GPS as both use a KL constraint to bound the change in policy. As such, the choice of approximating the Q-Function with time-dependent quadratic models (as done in MOTO) in order to solve the policy update instead of linearizing the system dynamics around the mean trajectory (as done in most trajectory optimization algorithms) is better isolated.
GPS and MOTO both use a time-dependent linear-Gaussian policy. In order to learn the linear model of the system dynamics, GPS reuses samples from different time-steps by learning a Gaussian mixture model on 
all the samples and uses this model as a prior to learn a joint Gaussian distribution $p(\s_t,\a_t,\s_{t+1})$ for every time-step. To single out the choice of linearizing the dynamics model or lack thereof from the different approaches to sample reuse, we give to both algorithm a high number of samples (200 and 400 rollouts per iteration for the double and quad link respectively) and bypass any form of sample reuse for both algorithms. \\

Fig. \ref{fig:swingTaska}.b compares GPS to two configurations of MOTO on the double-link swing up task. The same initial policy and step-size $\epsilon$ are used by both algorithm. 
However, we found that GPS performs better with a smaller initial variance, as otherwise actions quickly hit the torque limits making the dynamics modeling harder. 
Fig. \ref{fig:swingTaska}.b shows that even if the dynamics of the system are not linear, GPS manages to improve the policy return, and eventually finds a swing-up policy. 
The two configurations of MOTO have an entropy reduction constant $\beta_0$ of $.1$ and $.5$. 
The effect of the entropy constraint is similar to the one observed in the stochastic search domain by \citep{Abdolmaleki15}. 
Specifically, a smaller entropy reduction constant $\beta_0$ results in an initially slower convergence but ultimately 
leads to higher quality policies. In this particular task, MOTO with $\beta_0 = .1$ manages to slightly outperform GPS. \\

Next, GPS and MOTO are compared on the quad link swing-up task. We found this task to be significantly more challenging than the 
double link and to increase the difficulty further, soft joint limits are introduced on the three last joints in the following way: 
whenever a joint angle exceeds in absolute value the threshold $\frac{2}{3}\pi$, 
the desired torque of the policy is ignored in favor of a linear-feedback controller that aims at pushing 
back the joint angle within the constrained range. As a result, Fig. \ref{fig:swingTaskb}.b shows that GPS can barely improve its 
average return (with the torque limits set to 25, as in the double link task.) while MOTO performs significantly better. 
Finally, the torque limits are reduced even further but MOTO still manages to find a swing-up policy as demonstrated by Fig. \ref{fig:swingTaskb}.a. \\

In the last set of comparisons, the importance of each of the components of MOTO is assessed on the double link experiment. 
The number of rollouts per iteration is reduced to $M = 20$. Fig. \ref{fig:swingTaska}.c shows that: i) the entropy constraint provides an improvement on 
the quality of the policy in the last iterations in exchange of a slower initial progress, ii) importance sampling greatly helps 
in speeding-up the convergence and iii) the Monte-Carlo estimate of $\hat{Q}_i^t$ is not adequate for the 
smaller number of rollouts per iterations, which is further exacerbated by the fact that sample reuse of transitions from different time-steps is not possible with the Monte-Carlo estimate. \\

Finally, we explore on the double-link swing-up task several values of $M$, trying to find the balance between performing a small number of rollouts per iterations with a small step-size $\epsilon$ versus having a large number of rollouts for the policy evaluation that would allow to take larger update steps.
To do so, we start with an initial $M = 20$ and successively divide this number by two until $M = 5$. In each case, the entropy reduction constant is set such that,
for a similar number of rollouts, the entropy is reduced by the same amount, while we choose $\gamma'$, the discount of the state sample weights as $\gamma' = \gamma ^{M/M'}$ to yield again a similar sample decay after the 
same number of rollouts have been performed. Tuning $\epsilon$ was, however, more complicated and we tested several values on 
non-overlapping ranges for each $M$ and selected the best one. Fig. \ref{fig:swingTaskb}.c shows that, on the double link swing-up task, 
a better sample efficiency is achieved with a smaller $M$. However, the improvement becomes negligible from $M = 10$ to $M = 5$. 
We also noticed a sharp decrease in the number of effective samples when $M$ tends to 1. In this limit case, the complexity 
of the mixture policy $z^{1:i}$ in the denominator of the importance ratio  increases with the decrease of $M$ and might become a poor representation of the data set. 
Fitting a simpler state-action distribution that is more representative of the data can be the subject of future work 
in order to further improve the sample efficiency of the algorithm, which is crucial for applications on physical systems.

\subsection{Robot Table Tennis}
\begin{figure*}[t!]
    \centering
    \subfigure[]{\includegraphics[width=0.32\textwidth]{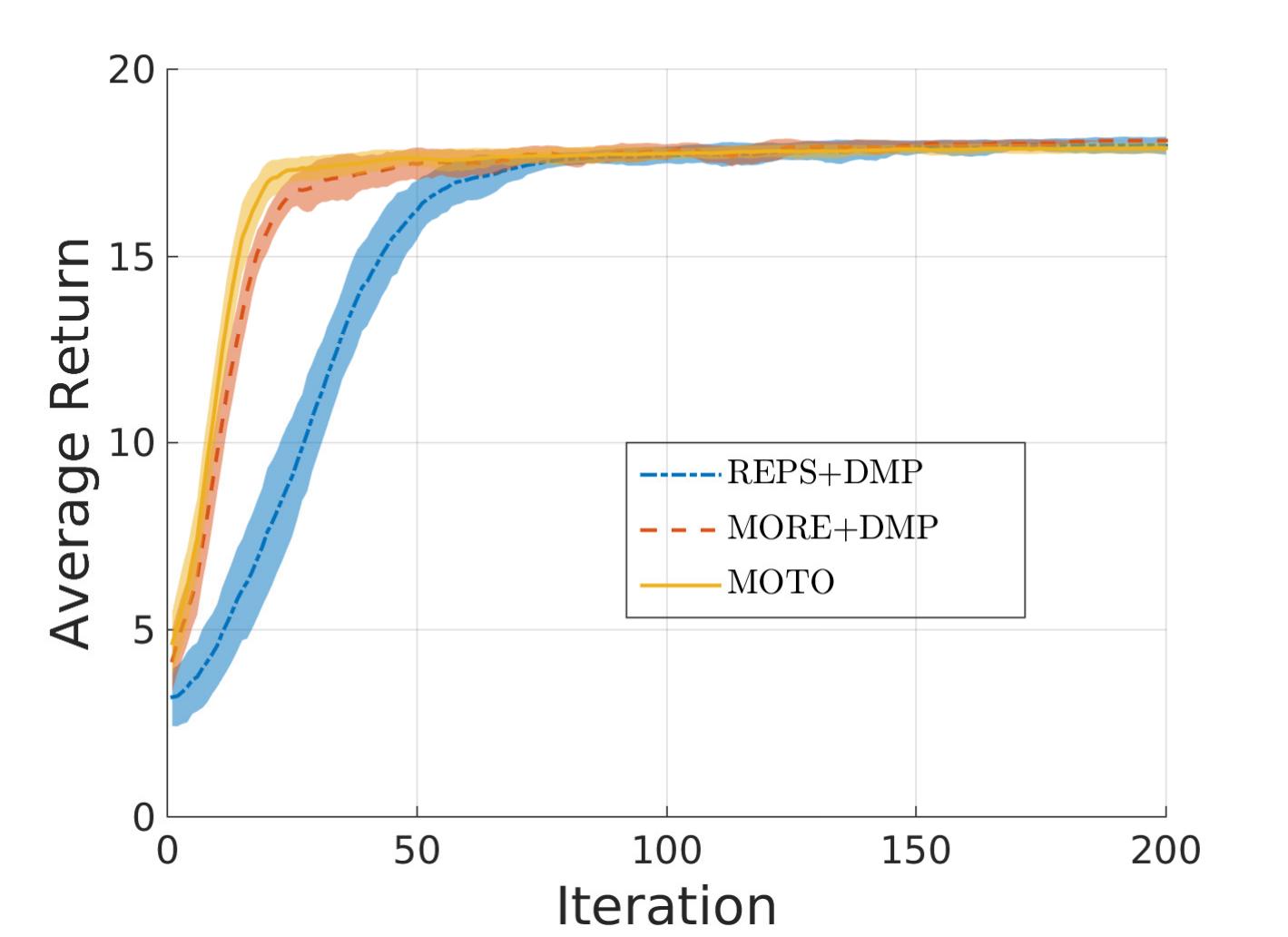}}
    ~ 
            \subfigure[]{\includegraphics[width=0.32\textwidth, height = 40mm]{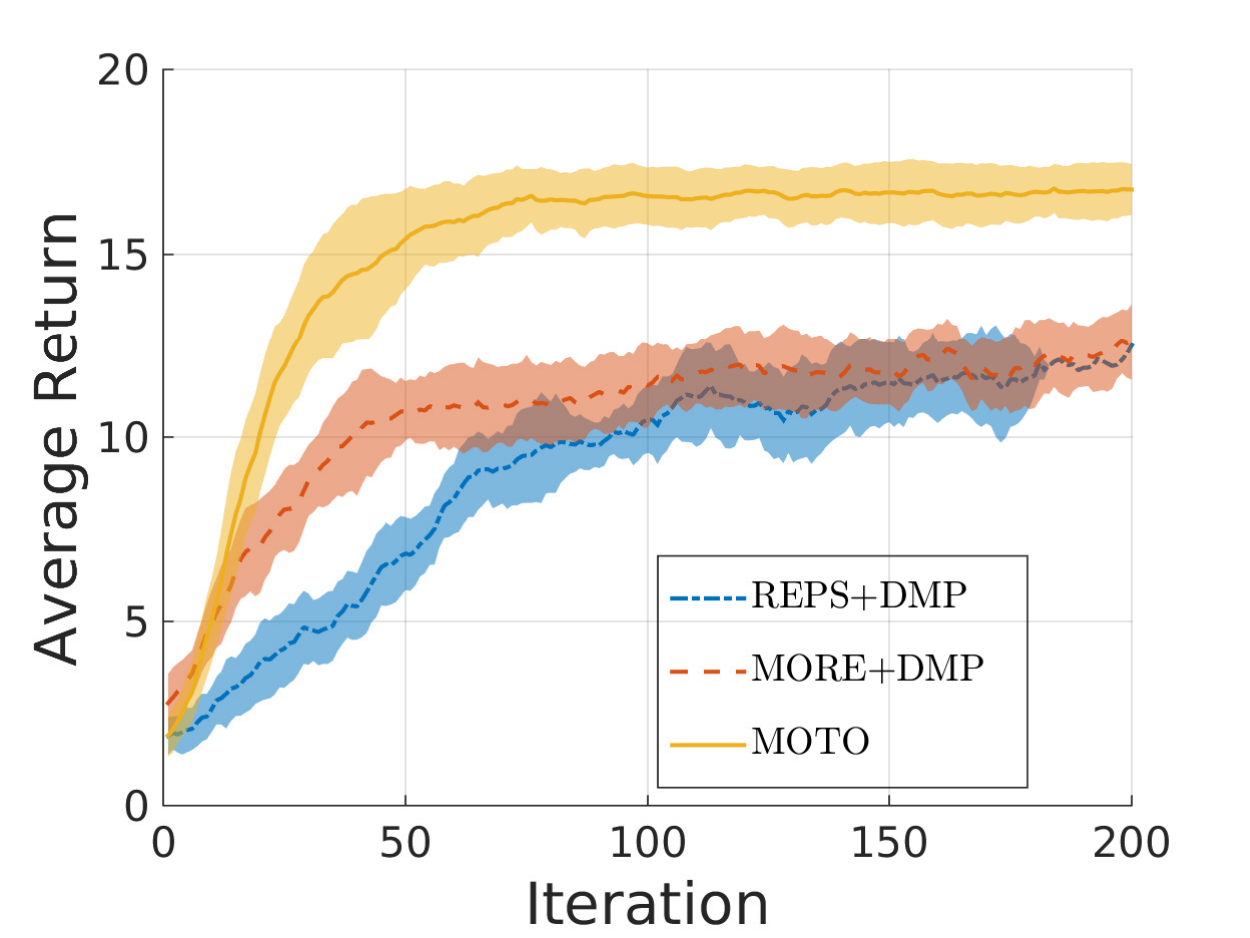}}
            \subfigure[]{\includegraphics[width=0.32\textwidth, height = 40mm]{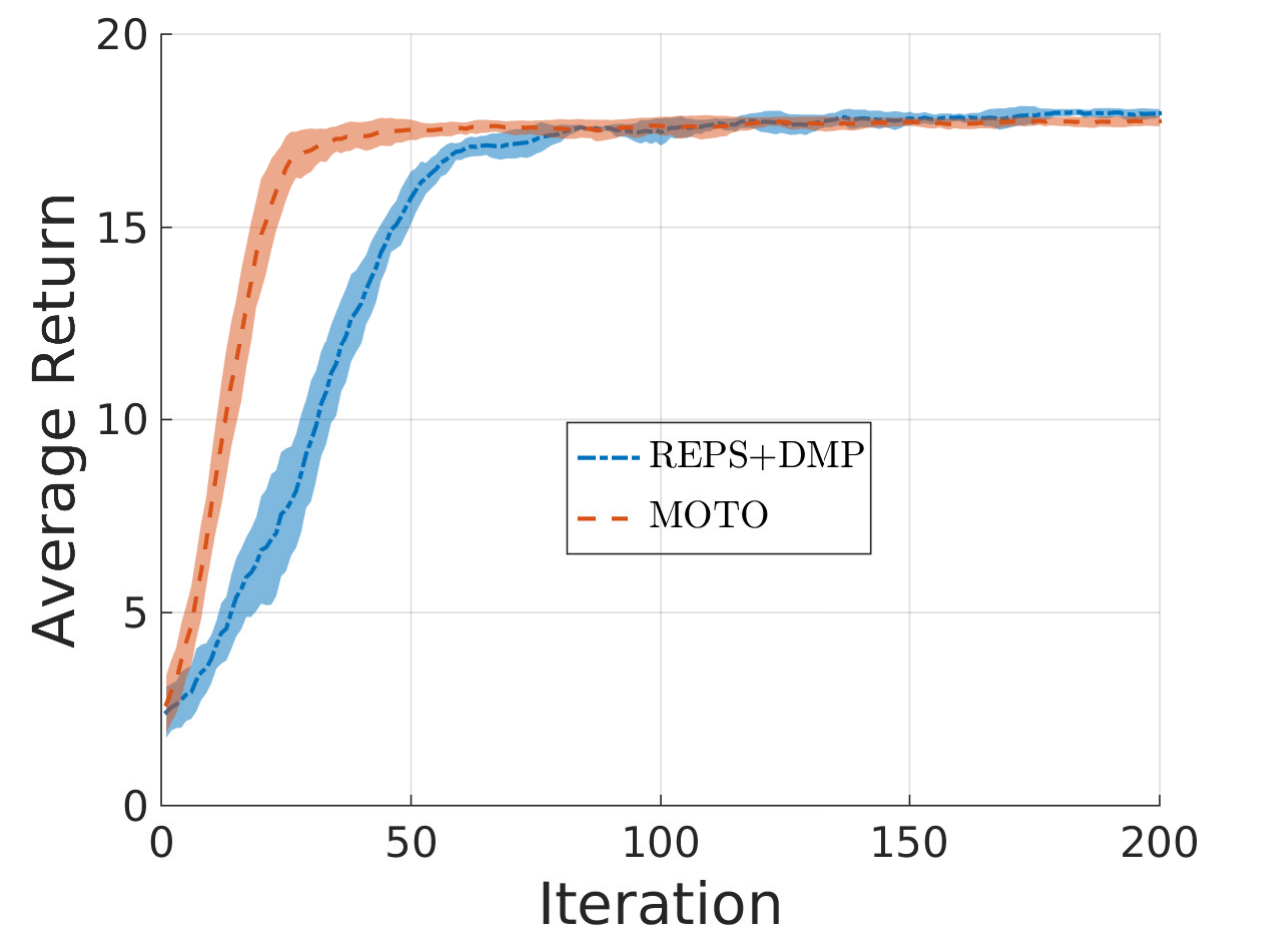}}
    ~ 
       
    \caption{ \small a) Comparison on the robot table tennis task with no noise on the initial velocity of the ball. b) Comparison on the robot table tennis task with Gaussian noise during the ball bounce on the table. c) Comparison on the robot table tennis task with initial velocity sampled uniformly in a 15cm range.}\label{fig:ttRes}
\end{figure*}

The considered robot table tennis task consists of a simulated robotic arm mounted on a floating base, having a racket on the end effector. The task of the robot is to return incoming balls using a forehand strike to the opposite side of the table (Fig. \ref{fig:ttSetting}). The arm has 9 degrees of freedom comprising the six joints of the arm and the three linear joints of the base allowing (small) 3D movement. Together with the joint velocities and the 3D position of the incoming ball, the resulting state space is of dimension $d_s = 21$ and the action space is of dimension $d_a = 9$ and consists of direct torque commands. 

\begin{figure}
\centering
\includegraphics[width=0.95\textwidth]{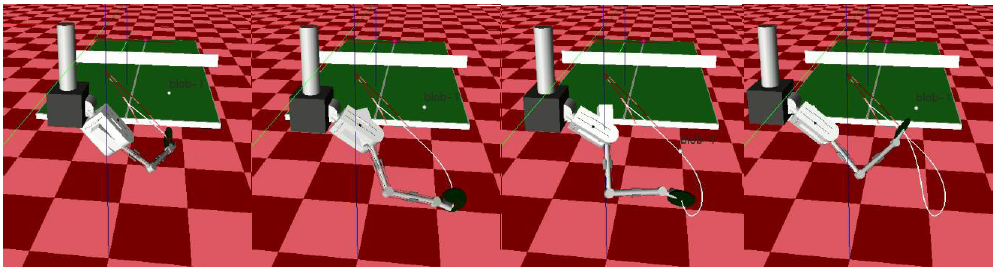}
	\caption{\small{Robot table tennis setting and a forehand stroke learned by MOTO upon a spinning ball.}}
    \label{fig:ttSetting}
\end{figure}

We use the analytical player of \cite{Muelling2011} to generate a single forehand stroke, which is subsequently used to learn from demonstration the initial policy $\pi^1$. The analytica player comprises a waiting phase (keeping the arm still), a preparation phase, a hitting phase and a return phase, which resets the arm to the waiting position of the arm. Only the preparation and the hitting phase are replaced by a learned policy. The total control time for the two learned phases is of 300 time-steps at 500hz, although for the MOTO algorithm we subsequently divide the control frequency by a factor of 10 resulting in a time-dependent policy of 30 linear-Gaussian controllers.

The learning from demonstration step is straightforward and only consists in averaging the torque commands of every 10 time-steps and using these quantities as the initial bias for each of the 30 controllers. Although this captures the basic template of the forehand strike, no correlation between the action and the state (e.g. the ball position) is learned from demonstration as the initial gain matrix $K$ for all the time-steps is set to the null matrix. Similarly, the exploration in action space is uninformed and initially set to the identity matrix.

Three settings of the task are considered, a noiseless case where the ball is launched with the same initial velocity, a varying context setting where the initial velocity is sampled uniformly within a fixed range and the noisy bounce setting where a Gaussian noise is added to both the x and y velocities of the ball upon bouncing on the table, to simulate the effect of a spin.

We compare MOTO to the policy search algorithm REPS 
 \citep{Kupcsik13} and the stochastic search algorithm MORE \citep{Abdolmaleki15} that shares a related information-theoretic update. Both algorithms will optimize the parameters of a Dynamical Movement Primitive (DMP) \citep{Ijspeert2003}. A DMP is a non-linear attractor system commonly used in robotics. The DMP is initialized from the same single trajectory and the two algorithm will optimize the goal joint positions and velocities of the attractor system. Note that the DMP generates a trajectory of states, which will be tracked by a linear controller using the inverse dynamics. While MOTO will directly output the torque commands and does not rely on this model.

Fig. \ref{fig:ttRes}.a and \ref{fig:ttRes}.c show that our algorithm converges faster than REPS and to a smaller extent than MORE in both the noiseless and the varying context setting. This is somewhat surprising since MOTO with its time-dependent linear policy have a much higher number of parameters to optimize than the 18 parameters of the DMP's attractor. However, the resulting policy in both cases is slightly less good than that of MORE and REPS. Note that for the varying context setting, we used a contextual variant of REPS that learns a mapping from the initial ball velocity to the DMP's parameters. MORE, on the other hand couldn't be compared in this setting. Finally, Fig. \ref{fig:ttRes}.b shows that our policy is successfully capable of adapting to noise at ball bounce, while the other methods fail to do so since the trajectory of the DMP is not updated once generated. 

\subsection{Comparison to Neural Network Policies}
Recent advances in reinforcement learning using neural network policies and supported by the ability of generating and processing large amounts of data allowed impressive achievements such as playing Atari at human level \citep{Mnih15} or mastering the game of Go~\citep{Silver16}. On continuous control tasks, success was found by combining trajectory optimization and supervised learning of a neural network policy \citep{Levine14}, or by directly optimizing the policy's neural network using reinforcement learning~\citep{Lillicrap15, Schulman15}. The latter work, side-stepping trajectory optimization to directly optimize a neural network policy raises the question as to whether the linear-Gaussian policies used in MOTO and related algorithms provide any benefit compared to neural network policies.

To this end, we propose to compare on the multi-link swing-up tasks of Sec. \ref{sec:swing}, MOTO learning a time-dependent linear-Gaussian policy to TRPO \citep{Schulman15} learning a neural network policy. We chose TRPO as our reinforcement learning baseline for its state-of-the-art performance and because of its similar policy update than that of MOTO~(both bound the KL between successive policies). Three variants of TRPO are considered while for MOTO, we refrain from using importance sampling~(Sec.~\ref{sec:importance}) since similar techniques such as off-policy policy evaluation can be used for TRPO.

First, MOTO is compared to a default version of TRPO using OpenAI's baselines implementation \citep{baselines} where TRPO optimizes a neural network for both learning the policy and the V-Function. Default parameters are used except for the KL divergence constraint where we set $\epsilon=.1$ for TRPO to match MOTO's setting. Note that because the rewards are time-dependent~(distance to the upright position penalized only for the last 20 steps, see Sec. \ref{sec:swing}) we add time as an additional entry to the state description. Time entry is in the interval $[0, 1]$ (current time-step divided by horizon $T$) and is fed to both the policy and V-Function neural networks. This first variant of TRPO answers the question: is there any benefit for using MOTO with its time-dependent linear-Gaussian policy instead of a state-of-the-art deep RL implementation with a neural network policy.

The second considered baseline uses the same base TRPO algorithm but replaces the policy evaluation using a neural network V-Function with the same policy evaluation used by MOTO (Sec. \ref{sec:policyEval}), back-propagating a quadratic V-Function. In this variant of TRPO the time-entry is dropped for the V-Function. This second baseline better isolates the policy update, which is the core of both algorithms, from the learning of the V-Function which could be interchanged. 

Finally, we consider a third variant of TRPO that uses both the quadratic V-Function and a time-dependent linear-Gaussian policy with diagonal covariance matrix (standard formulation and implementation of TRPO does not support full covariance exploration noise). The time entry is dropped for both the V-Function and the policy in this third baseline. While both algorithms bound the KL divergence between successive policies, there are still a few differentiating factors between this third baseline and MOTO. First, TRPO bounds the KL of the whole policy while MOTO solves a policy update for each time-step independently (but still results in a well-founded approximate policy iteration algorithm as discussed in Sec.~\ref{sec:theo}). In practice the KL divergence upon update for every time-step for MOTO is often equal to $\epsilon$ and hence both MOTO and TRPO result in the same KL divergence of the overall policy~(in expectation of the state distribution) while the KL divergence of the sub-policies (w.r.t. the time-step) may vary. Secondly, MOTO performs a quadratic approximation of the Q-Function and solves the policy update exactly while TRPO performs a quadratic approximation of the KL constraint and solves the policy update using conjugate gradient descent. TRPO does not solve the policy update in closed form because it would require a matrix inversion and the matrix to invert has the dimensionality of the number of policy parameters. In contrast, MOTO can afford the closed form solution because the matrix to invert has the dimensionality of the action space which is generally significantly smaller than the number of policy parameters. 

Fig. \ref{fig:trpo} shows the learning performance of MOTO and three TRPO variants on the double link and quadruple link swing-up tasks~(Sec.~\ref{sec:swing}). In both tasks MOTO outperforms all three TRPO variants albeit when TRPO is combined with the quadratic V-Function (second variant), it initially outperforms MOTO on the double link swing-up task. The quadratic V-Function befits these two tasks in particular and the quadratic regulation setting more generally because the reward is a quadratic function of the state-action pair~(here the negative squared distance to the upright position and a quadratic action cost). However, MOTO makes better use of the task's nature and largely outperforms the third variant of TRPO despite having a similar policy evaluation step and using the same policy class. In conclusion, while neural networks can be a general purpose policy class demonstrating success on a wide variety of tasks, on specific settings such as on quadratic regulator tasks, trajectory-based policy optimization is able to outperform deep RL algorithms. MOTO in particular, which does not rely on a linearization of the dynamics around the mean trajectory is able to handle quadratic reward problems with highly non-linear dynamics such as the quadruple link swing-up task and outperform state-of-the-art trajectory optimization algorithms (Sec.~\ref{sec:swing}) as a result.

\begin{figure*}[t!]
    \centering
    \subfigure{\includegraphics[width=0.48\textwidth]{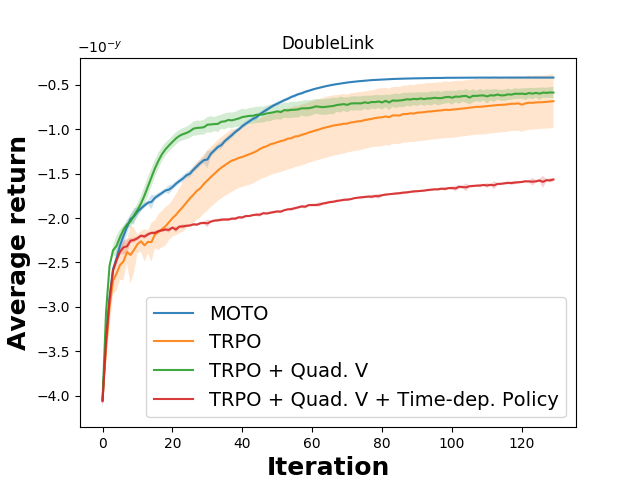}}
    ~ 
    \subfigure{\includegraphics[width=0.48\textwidth]{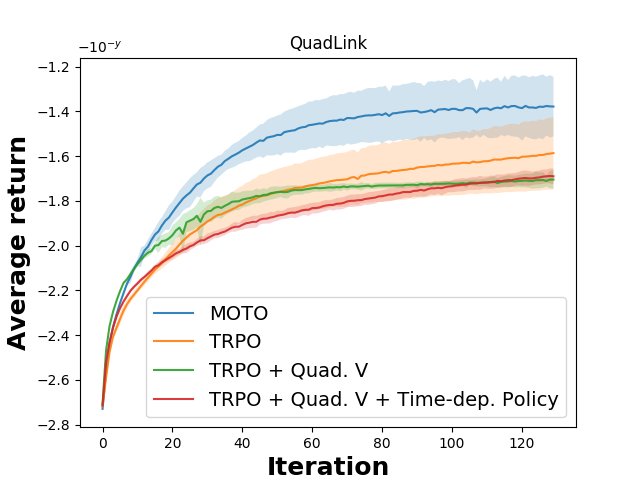}}
    ~ 
       
    \caption{ \small Comparisons on multi-link swing-up tasks between MOTO and TRPO. TRPO uses a neural network policy and V-Function (default) or a quadratic V-Function and a time-dependent linear-Gaussian policy as in MOTO. Quadratic V-Function is a good fit for such tasks and allows MOTO to outperform neural network policies on the double and quadruple link swing-up tasks. Rewards of the original task divided by $1\text{e}4$ to accomodate with the neural network V-Function. Plots averaged over 11 independent runs.}\label{fig:trpo}
\end{figure*}

\section{Conclusion}
We proposed in this article MOTO, a new trajectory-based policy optimization algorithm that does not rely on a linearization of the dynamics. Yet, an efficient policy update could be derived in closed form by locally fitting a quadratic Q-Function. We additionally conducted a theoretical analysis of the constrained optimization problem solved during the policy update. We showed that the upper bound on the expected KL between successive policies leads to only a small drift in successive state distributions which is a key property in the approximate policy iteration scheme.

The use of a KL constraint is widely spread including in other trajectory optimization algorithms. The experiments demonstrate however that our algorithm has an increased robustness towards non-linearities of the system dynamics when compared to a closely related trajectory optimization algorithm. It appears as such that the simplification resulting from considering a local linear approximation of the dynamics is more detrimental to the overall convergence of the algorithm than a local quadratic approximation of the Q-Function. 

On simulated robotics tasks, we demonstrated the merits of our approach compared to direct policy search algorithms that optimize commonly used low dimensional parameterized policies. The main strength of our approach is its ability to learn reactive policies capable of adapting to external perturbations in a sample efficient way. However, the exploration scheme of our algorithm based on adding Gaussian noise at every time-step is less structured than that of low dimensional parameterized policies and can be harmful to the robot. One of the main addition that would ease the transition from simulation to physical systems is thus to consider the safety of the exploration scheme of the algorithm. On a more technical note, and as the V-Function can be of any shape in our setting, the use of a more complex function approximator such as a deep network can be considered in future extensions to allow for a more refined bias-variance trade-off. 

\acks{This work was supported by the DFG Project LearnRobotS under the SPP 1527 Autonomous Learning.}

\newpage
\appendix
\section{Dual Function Derivations}
\label{sec:app_grad}
Recall the quadratic form of the Q-Function $\tilde{Q}_t(\s,\a)$ in the action $a$ and state $s$
\begin{align*}
\tilde{Q}_t(\s,\a) = \frac{1}{2}\a^TQ_{aa}\a + \a^TQ_{as}\s+\a^T\boldsymbol{q}_{a}+q(\s).
\end{align*}
The new policy $\pi'_t(\a|\s)$ solution of the constrained maximization problem is again of linear-Gaussian form and given by
\[\pi'_t(\a|\s) = \Normal(\a|FL\s+F\boldsymbol{f}, F(\eta^*+\omega^*)),\] 
such that the gain matrix, bias and covariance matrix of $\pi'_t$ are function of matrices $F$ and $L$ and vector $\boldsymbol{f}$ where
\begin{align*}
&F = (\eta^* \Sigma_t^{-1}-Q_{aa})^{-1}, & & L = \eta^*\Sigma_t^{-1}K_t+Q_{as},\\ & \boldsymbol{f} = \eta^*\Sigma_t^{-1}\boldsymbol{k}_t+\boldsymbol{q}_a,
\end{align*}
with $\eta^*$ and $\omega^*$ being the optimal Lagrange multipliers related to the KL and entropy constraints, obtained by minimizing the dual function 
\begin{multline*}
g_t(\eta, \omega) = \eta\epsilon - \omega\beta + (\eta+\omega)\int \tilde{\rho}_t(\s) \log\left( \int \pi(\a|\s)^{\eta/(\eta+\omega)}\exp\left(\tilde{Q}_t(\s,\a)/(\eta+\omega)\right) \right).
\end{multline*}

From the quadratic form of $\tilde{Q}_t(\s,\a)$ and by additionally assuming that the state distribution is approximated by $\tilde{\rho}_t(\s) = \Normal(\s|\boldsymbol{\mu}_s, \Sigma_s)$, the dual function simplifies to
\[g_t(\eta, \omega) = \eta\epsilon-\omega\beta+\boldsymbol{\mu}^T_s M\boldsymbol{\mu}_s + \tr(\Sigma_sM) + \boldsymbol{\mu}^T_s \boldsymbol{m} + m_0,\]
where $M$, $\boldsymbol{m}$ and $m_0$ are defined by
\begin{multline*}
M = \frac{1}{2}\left(L^TFL - \eta K_t^T \Sigma_t^{-1} K_t\right),\ \boldsymbol{m} = L^TF\boldsymbol{f} - \eta K_t^T \Sigma_t^{-1} \boldsymbol{k}_t,\\
m_0 = \frac{1}{2}(\boldsymbol{f}^TF\boldsymbol{f} - \eta  \boldsymbol{k}^T_t\Sigma_t^{-1} \boldsymbol{k}_t - \eta \log |2\pi\Sigma_t| + (\eta+\omega)\log|2\pi(\eta+\omega)F|).
\end{multline*}
The convex dual function $g_t$ can be efficiently minimized by gradient descent and the policy update is performed upon the computation of $\eta^*$ and $\omega^*$. The gradient w.r.t. $\eta$ and $\omega$ is given by\footnote{cst, lin, quad, $F$, $L$ and $\boldsymbol{f}$ all depend on $\eta$ and $\omega$. We dropped the dependency from the notations for compactness. $d_a$ is the dimensionality of the action.}

\begin{align*}
\frac{\partial g_t(\eta, \omega)}{\partial \eta} ={}& \text{cst} + \text{lin} + \text{quad}\\
\begin{split}
\text{cst} ={}& \epsilon - \frac{1}{2} \left( \boldsymbol{k}_t - F\boldsymbol{f}  \right)^T \Sigma_t^{-1}
\left( \boldsymbol{k}_t - F\boldsymbol{f}  \right) - \frac{1}{2}[\log|2\pi\Sigma_t| - \log|2\pi(\eta+\omega)F| \\& + (\eta+\omega)\tr(\Sigma_t^{-1}F) - d_a ].
\end{split}\\
\text{lin} ={}& \left((K_t - FL) \boldsymbol{\mu}_s\right)^T\Sigma_t^{-1}(F\boldsymbol{f}-\boldsymbol{k}_t).\\
\text{quad} = {}& \boldsymbol{\mu}_s^T(K_t+FL)^T\Sigma_t^{-1}(K_t+FL)\boldsymbol{\mu}_s +\tr(\Sigma_s(K_t+FL)^T\Sigma_t^{-1}(K_t+FL))\\
\frac{\partial g_t(\eta, \omega)}{\partial \omega} ={}& -\beta + \frac{1}{2} (d_a + \log|2\pi(\eta+\omega)F|).
\end{align*}

\section{Bounding the Expected Policy KL Under the Current State Distribution}
\label{sec:bounds}
Let the state distributions and policies be parameterized as following:  $p_t(\s) = \Normal(\s|\boldsymbol{\mu}_p, \Sigma_p)$, $q_t(\s) = \Normal(\s|\boldsymbol{\mu}_q, \Sigma_q)$, $p_{t}(\a|\s) = \Normal(\a|K\s+\boldsymbol{b}, \Sigma)$ and $q_{t}(\a|\s) = \Normal(\a|K'\s+\boldsymbol{b}', \Sigma')$. The change of the state distribution in the expected KL constraint of our policy update, given by $\EE_{\s\sim q_{t}}[\KLM{p_{t}(.|\s)}{q_{t}(.|\s)}]$ from state distribution $q_t$ to $p_t$ will only affect the part of the KL that depends on the state.\\

We give as a reminder the general formula for the KL between two Gaussian distributions $l = \Normal(\boldsymbol{\mu}, \Sigma)$ and $l' = \Normal(\boldsymbol{\mu}', \Sigma')$ 
\[\KLM{l}{l'} = \frac{1}{2}\left(\text{tr}(\Sigma^{'-1}\Sigma) + (\boldsymbol{\mu}-\boldsymbol{\mu}')^T \Sigma^{-'1}(\boldsymbol{\mu}-\boldsymbol{\mu}') - dim + \log\frac{|\Sigma'|}{|\Sigma|} \right).\]\\

For the linear-Gaussian policies, and since only the mean of the policies depend on the state, the change of state distribution in the expected KL will only affect the term \[(K\s-K'\s)^T\Sigma'(K\s-K'\s) = \s^TM\s,\]
with p.s.d. matrix $M = (K-K')^T\Sigma'(K-K')$. Thus it suffices to bound the expectation $\int p_t(\s) \s^TM\s$ since the rest of the KL terms are already bounded by $\epsilon$, yielding
\begin{align*}
\EE_{s\sim p_{t}}[\KLM{p_{t}(.|\s)}{q_{t}(.|\s)}] &\leq \epsilon + \frac{1}{2}\int p_t(\s) \s^TM\s,\\
&= \epsilon + \frac{1}{2}\left(\boldsymbol{\mu}_pM\boldsymbol{\mu}_p+\tr(M\Sigma_p)\right),
\end{align*}
where we exploited the Gaussian nature of $p_t$ in the second line of the equation. We will now bound both $\boldsymbol{\mu}_pM\boldsymbol{\mu}_p$ and $\tr(M\Sigma_p)$. First, note that for any two p.d. matrices $\Sigma$ and $\Sigma'$ we have 
\begin{align}
\label{posKLCov}
\text{tr}(\Sigma^{'-1}\Sigma) + - d_s + \log\frac{|\Sigma'|}{|\Sigma|} \geq 0.
\end{align}
This immediately follows from the non-negativity of the KL. Since, if for some $\Sigma$ and $\Sigma'$, eq. \eqref{posKLCov} is negative then the KL for two Gaussian distributions having $\Sigma$ and $\Sigma'$ as covariance matrices and sharing the same mean would be negative which is not possible. Hence it also follows that
\begin{align}
(\boldsymbol{\mu}_q-\boldsymbol{\mu}_p)^T \Sigma_q^{-1}(\boldsymbol{\mu}_q-\boldsymbol{\mu}_p) \leq 2\epsilon_t,\label{boundedMean}
\end{align}
from the bounded KL induction hypothesis between $p_t$ and $q_t$.\\

For the expected policy KL, since the part that does not depend on $s$ is positive as in eq. \eqref{posKLCov}, it can thus be dropped out yielding
\begin{align} 
\EE_{s\sim q_{t}}[\KLM{p_{t}(.|\s)}{q_{t}(.|\s)}] \leq \epsilon &\Rightarrow \int q_t(\s)\s^TM\s \leq 2\epsilon,\notag\\
&\Rightarrow \boldsymbol{\mu}_qM\boldsymbol{\mu}_q+\tr(M\Sigma_q) \leq 2\epsilon.\label{boundedQ}
\end{align}

Also note that for any p.s.d. matrices $A$ and $B$, $\tr(AB) \geq 0$. Letting $\boldsymbol{x} = \boldsymbol{\mu}_p - \boldsymbol{\mu}_q$, we have 
\begin{align*}
\boldsymbol{x}^TM\boldsymbol{x} &= \tr(\boldsymbol{x} \boldsymbol{x}^T M),\\
&= \tr(\Sigma_q^{-1}\boldsymbol{x} \boldsymbol{x}^T M\Sigma_q),\\
&\leq \tr(\Sigma_q^{-1}\boldsymbol{x} \boldsymbol{x}^T) \tr(M\Sigma_q),\\ 
&\leq 4\epsilon_t \epsilon.
\end{align*}
Third line is due to Cauchy-Schwarz inequality and positiveness of traces while the last  one is from eq. \eqref{boundedMean} and \eqref{boundedQ}.
Finally, from the reverse triangular inequality, we have  
\begin{align*}
\boldsymbol{\mu}_pM\boldsymbol{\mu}_p &\leq \boldsymbol{x}^TM\boldsymbol{x} + \boldsymbol{\mu}_qM\boldsymbol{\mu}_q,\\
&\leq 2\epsilon(1+2\epsilon_t),
\end{align*}
Which concludes the bounding of $\boldsymbol{\mu}_pM\boldsymbol{\mu}_p$.\\

To bound $\tr(M\Sigma_p)$ we can write 
\begin{align*}
\tr(M\Sigma_p) &= \tr(M\Sigma_q\Sigma_q^{-1}\Sigma_p),\\
&\leq\tr(M\Sigma_q)\tr(\Sigma_q^{-1}\Sigma_p).
\end{align*}
We know how to bound $\tr(M\Sigma_q)$ from Eq. \eqref{boundedQ}. While $\tr(\Sigma_q^{-1}\Sigma_p)$ appears in the bounded KL between state distributions. Bounding $\tr(\Sigma_q^{-1}\Sigma_p)$ is equivalent to solving $\max \sum \lambda_i$ under constraint $\sum \lambda_i - d_s - \sum\log\lambda_i \leq 2\epsilon_t$, where the $\{\lambda_i\}$ are the eigenvalues of $\Sigma_q^{-1}\Sigma_p$. For any solution $\{\lambda_i\}$, we can keep the same optimization objective using equal $\{\lambda_i'\}$ where for each $i$, $\lambda_i' = \bar{\lambda} = \sum \lambda_i/d_s$ is the average lambda. This transformation will at the same time reduce the value of the constraint since $-d_s\log\bar{\lambda} \leq -\sum\log\lambda_i$ from Jensen's inequality. Hence the optimum is reached when all the $\lambda_i$ are equal, and the constraint is active (i.e. $d_s \bar{\lambda} -d_s - d_s\log{\bar{\lambda}} = 2\epsilon_t$). Finally, the constraint is at a minimum for $\bar{\lambda} = 1$, hence $\bar{\lambda} > 1$. The maximum is reached at
\begin{align}
&d_s \bar{\lambda} -d_s - d_s\log{\bar{\lambda}} = 2\epsilon_t \notag\\
\Leftrightarrow &\bar{\lambda} - \log{\bar{\lambda}} = \frac{2\epsilon_t}{d_s} + 1 \notag\\
\Rightarrow &\bar{\lambda} \leq \left(\frac{2\epsilon_t}{d_s} + 1\right) \frac{e}{e-1} \notag\\
\Rightarrow &\tr(\Sigma_q^{-1}\Sigma_p) \leq 4 \epsilon_t + 2d_s \label{boundedTr}
\end{align} 
The equation in the second line has a unique solution ($f(\lambda) = \lambda - \log \lambda$ is a strictly increasing function for $\lambda > 1$) for which no closed form expression exists. We thus lower bound $f$ by $g(\lambda) = \frac{e-1}{e}\lambda$ and solve the equation for $g$ which yields an upper bound of the original equation that is further simplified in the last inequality.\\
Eq. \eqref{boundedQ} and \eqref{boundedTr} yield $\tr(M\Sigma_p) \leq 2\epsilon (4\epsilon_t + 2d_s)$ and grouping all the results yields
\begin{align*}
\EE_{s\sim p_{t}}[\KLM{p_{t}(.|\s)}{q_{t}(.|\s)}] &\leq 2\epsilon \left(3 \epsilon_t +  d_s + 1\right)
\end{align*}
\newpage
\vskip 0.2in
\bibliography{papers}

\end{document}